\documentclass[11pt]{article}

\usepackage[preprint]{acl}

\usepackage{times}
\usepackage{latexsym}

\usepackage[T1]{fontenc}

\usepackage[utf8]{inputenc}

\usepackage{microtype}

\usepackage{inconsolata}

\usepackage{graphicx}
\usepackage[utf8]{inputenc}
\usepackage[T1]{fontenc}
\usepackage{hyperref}
\usepackage{url}
\usepackage{graphicx}
\usepackage{booktabs}
\usepackage{array}
\usepackage{tabularx}
\usepackage{amsmath}
\usepackage{amssymb}
\usepackage{amsthm}
\usepackage{amsfonts}
\usepackage{nicefrac}
\usepackage{microtype}
\usepackage{xcolor}
\usepackage[table]{xcolor}
\definecolor{sealblue}{RGB}{230, 236, 247}
\usepackage{wrapfig}
\usepackage{graphicx}
\usepackage{subcaption}
\usepackage{enumitem}
\newtheorem{definition}{Definition}
\newtheorem{proposition}{Proposition}
\newtheorem{lemma}{Lemma}

\newcommand{\method}{\sophia{}}
\newcommand{\sophia}{\textbf{\textsc{Sophia}}}

\title{Can We Break LLMs Out of Self-Loops? Fine-Grained Reasoning Control with Activation Steering}

\author{
  \textbf{Sheldon Yu\textsuperscript{1}},
  \textbf{Tong Yu\textsuperscript{2}},
  \textbf{Xunyi Jiang\textsuperscript{1}},
  \textbf{Rohan Surana\textsuperscript{1}},
  \textbf{Gagan Mundada\textsuperscript{1}},
\\
  \textbf{Sungchul Kim\textsuperscript{2}},
  \textbf{Lina Yao\textsuperscript{3}},
  \textbf{Julian McAuley\textsuperscript{1}},
  \textbf{Junda Wu\textsuperscript{1}}
\\
  \textsuperscript{1}UC San Diego \quad
  \textsuperscript{2}Adobe Research \quad
  \textsuperscript{3}University of New South Wales
\\
  \small\texttt{\{ziy040,xuj003,rsurana,gmundada,jmcauley,juw069\}@ucsd.edu}
\\
  \small\texttt{\{tyu,sukim\}@adobe.com} \quad \texttt{lina.yao@unsw.edu.au}
}

\begin{document}
\maketitle
\begin{abstract}
Extended reasoning is now standard for frontier Large Language Models (LLMs), yet the trajectories they produce remain largely uncontrolled, and inference-time methods for shaping them act mostly at the input or token level. We take a transition-level view: treating each reasoning trace as a sequence of latent states, we show that failing traces tend to get stuck in \emph{self-loops} that re-enter the same state without progress. We then ask a narrow controllability question; can an inference-time activation intervention drive the model \emph{out} of its current latent state on the next step? and answer it with \sophia{}: \textbf{S}teering \textbf{O}f reasoning \textbf{P}rocesses via \textbf{H}idden-state \textbf{I}ntervention and \textbf{A}ctivations. Training-free, \sophia{} extracts per source state a contrastive residual-stream direction that separates steps exiting the state from steps that stay, and applies it through an online self-loop detector. Across three models and four datasets, \sophia{} raises the loop-exit rate over greedy decoding, and the effect is \emph{direction-specific}: norm-matched random offsets do not reproduce it and sign-flipping suppresses it. We establish latent reasoning states as a causally manipulable interface and self-loop exit as a concrete control primitive.
\end{abstract}

\begin{figure*}[t]
    \centering
    \includegraphics[width=0.98\linewidth]{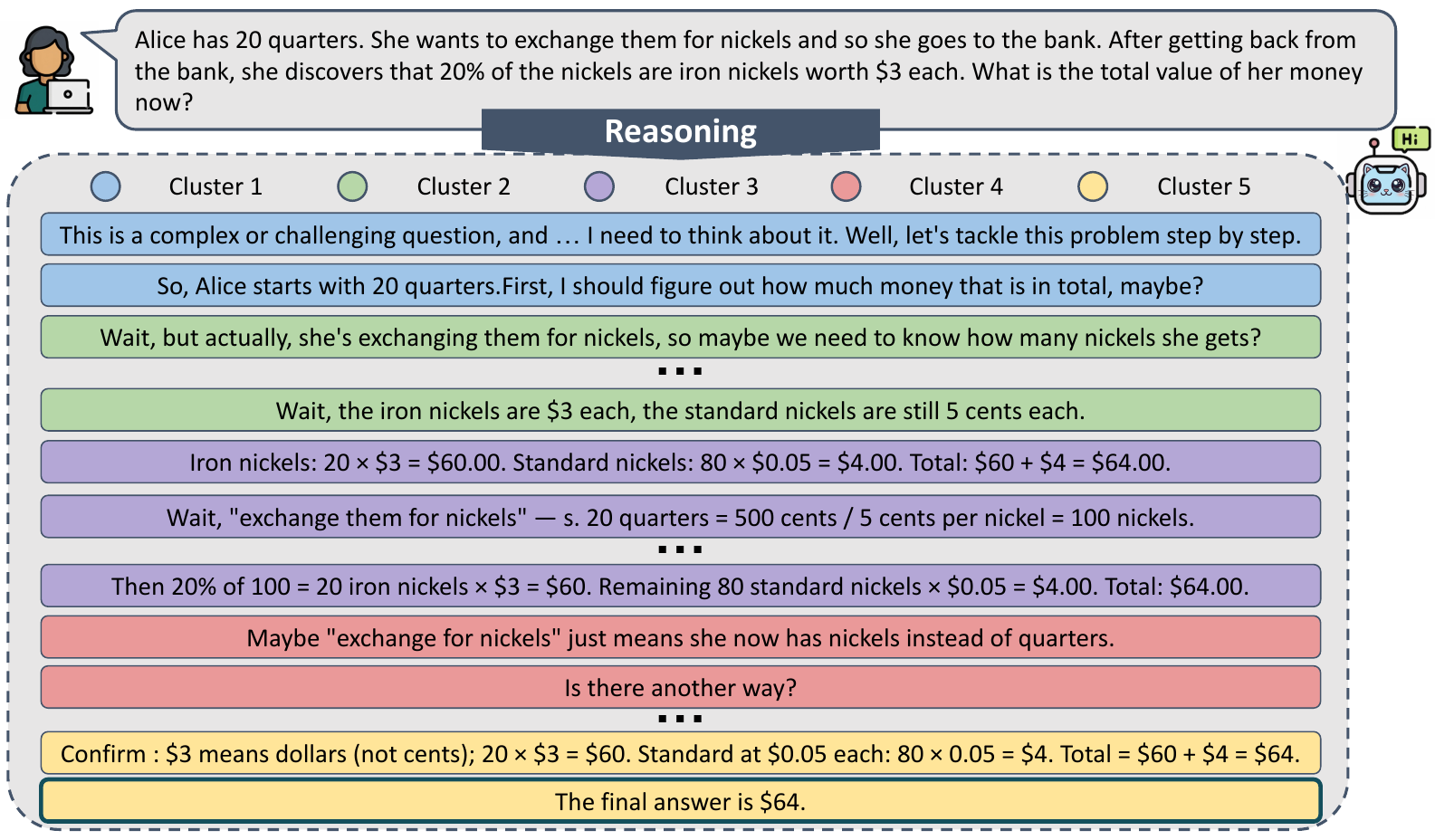}
    \caption{\textbf{Self-loops in LRM reasoning traces.}
The model computes the correct value early (purple, \$64) but cycles
through redundant verification and re-questioning steps (red) before
committing, wasting budget on work it has already done. Box colors
denote latent reasoning states $K = 5$ that are induced unsupervisedly from
activations rather than assigned from a hand-defined thought
taxonomy~\cite{chen2025seal}, with transitions discovered automatically
from generated traces. This makes failure modes like the loop above
targetable as transition events at finer granularity than token- or
prompt-level control allows.
}
\end{figure*}

\section{Introduction}
Large Reasoning Models (LRMs), such as OpenAI’s o1-series  have demonstrated compelling capabilities for complex reasoning tasks via the extended
chain-of-thought \cite{wei2022chain,kojima2022large,surana2026generate,wu2025doc,wu2024decot} reasoning mechanism. By generating intermediate reasoning trace, a model can allocate additional computation to arithmetic \cite{ahn2024large, luo2025wiz}, planning \cite{wang2024descri, valmeekam2023plann}, and multi-hop question answering \cite{zhong-etal-2024-debug}. These models are designed to explore diverse reasoning strategies, reflect on their decisions, and iteratively refine solutions. The same mechanism, however, can also waste computation.
Long traces often include repeated verification \cite{marjano}, injected prompts \cite{wang2026cacheprune}, stuck in self-loops \cite{keskar2019ctrl, welleck2019neural}, or late detours that increase latency without improving the final answer \cite{arcuschin2025c}.
For reasoning models, the question is therefore not only how to elicit more thinking, but how to keep the reasoning trajectory moving through productive states. This paper takes a transition-level view of that problem. We treat each reasoning 
trace as a sequence of latent reasoning states rather than only as a sequence of 
tokens. At step $t$, the prefix $\xi_t$ is mapped to a discrete state $z_t$, and 
the next generated step induces a transition $z_t \to z_{t+1}$. Under this view, the most concrete and damaging form of unproductive reasoning is the \emph{self-loop}: the model repeatedly transitions back to the same state, 
re-verifying, restating, or revisiting prior work without producing new information, and frequently spends redundant computation before reaching an
answer. Moreover, overlong reasoning is strongly correlated with incorrect answers
(Figure~\ref{fig:token-counts}), and self-loops are a dominant transition pattern we
characterize in Section~\ref{sec:transition-analysis}. This perspective suggests a more precise target for inference-time
control. Instead of steering the whole response toward a global attribute, we can ask a
narrower and more testable question: whether a model can be driven \emph{out} of its
current reasoning state on the next step and, in particular, whether self-loops can be
detected and broken online, mid-trace.

Prior work controls reasoning either by training the policy, e.g.\ L1
\cite{aggarwal2025l1} for length, SCoRe \cite{kumar2024score} for
self-correction, and GRPO- and DPO-style policy optimization
\cite{mundada2026wsgrpo,xiong2026rrpo,yu2026failure,surana2026massdpo}, which require retraining and cannot intervene once the model is
deployed---or, at inference time, at coarse granularity through surface-token edits
such as thought-switching penalties \cite{wang2023selfconsistency} and prompt
insertions \cite{zhang2025smartswitch} (see Related Work). Yet standard steering vectors are typically state-agnostic and are applied uniformly
throughout generation, regardless of where the model is in its reasoning trajectory
\cite{turner2023steering,tan2024analysing,huang2026amps}. 

To examine whether such state-agnostic control suits CoT reasoning, we analyze the
transition graph of reasoning traces and find that it is highly structured
(Figure~\ref{fig:cluster-structure}): self-loops and a small set of dominant
transitions account for most observed steps, and the accompanying t-SNE projections
show that steps of different reasoning states occupy distinct regions of activation
space. This structure exposes a mismatch with state-agnostic steering: the signal that
breaks a verification loop need not match the signal that consolidates an answer or
continues a decomposition. The signal separating one transition from another is
instead state-dependent, which motivates indexing the intervention by the model's
current reasoning state.

Our goal is therefore to break self-loops online with an intervention conditioned on
the state the model is currently in. To this end, we propose \sophia{}:
\textbf{S}teering \textbf{O}f reasoning \textbf{P}rocesses via \textbf{H}idden-state
\textbf{I}ntervention and \textbf{A}ctivations, a training-free framework that builds a
\emph{source-state-indexed} bank of exit directions and applies the one matching the
current state. \sophia{} segments model-generated traces into steps and fits a latent
state abstraction over their activations.

\begin{figure}[t]
    \centering
    \includegraphics[width=\linewidth]{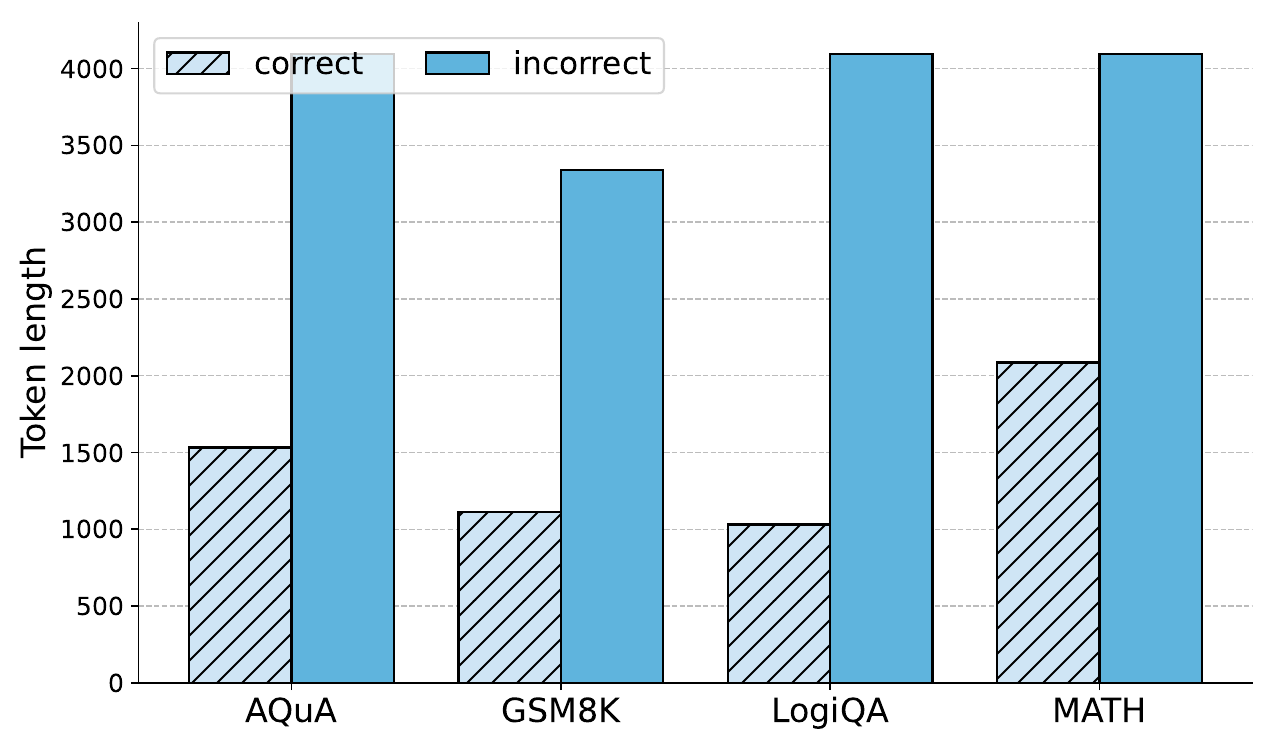}
    \caption{Average reasoning-trace token length for \emph{correct} (hatched) vs.\ \emph{incorrect} (solid) responses across \textsc{aqua}, \textsc{gsm8k}, \textsc{logiqa}, and \textsc{math}. Incorrect responses consume substantially more tokens on every dataset, suggesting that longer traces correlate with failure rather than additional useful reasoning.}
    \label{fig:token-counts}
\end{figure}

For each source state it extracts a contrastive residual-stream direction separating
steps that \emph{exit} the state (crossers) from steps that stay (stayers). At inference
time an online controller infers the current state, detects self-loops from the
transition structure, and---without external supervision---patches the residual stream
along that state's exit direction to trigger a corrective exit. We evaluate this
self-loop-exit setting; extending to target-directed transitions is left to future work.

Our contributions are:
\begin{itemize}[leftmargin=1.5em,itemsep=0.05em]
    \item We reframe inference-time reasoning control as steering latent
    reasoning-state transitions, taking the model's current reasoning state as the unit
    of intervention and self-loop exit as the first concrete control target.

    \item We characterize the transition structure of reasoning traces both
    semantically and statistically, identifying systematic patterns that distinguish
    correct from incorrect trajectories, with self-loops emerging as a prominent
    failure signature of overlong, incorrect traces.

    \item We show that a contrastive residual-stream direction extracted per source
    state can causally raise the probability that the model exits that state on the
    next step---breaking self-loops at inference time without modifying model
    weights---and that this effect is direction-specific rather than a generic
    perturbation.

    \item We package this mechanism into \sophia{}, a training-free online controller
    that detects self-loops mid-trace and applies the corresponding source-state exit
    direction.
\end{itemize}

\begin{figure*}[t]
  \centering
  \includegraphics[width=0.24\textwidth]{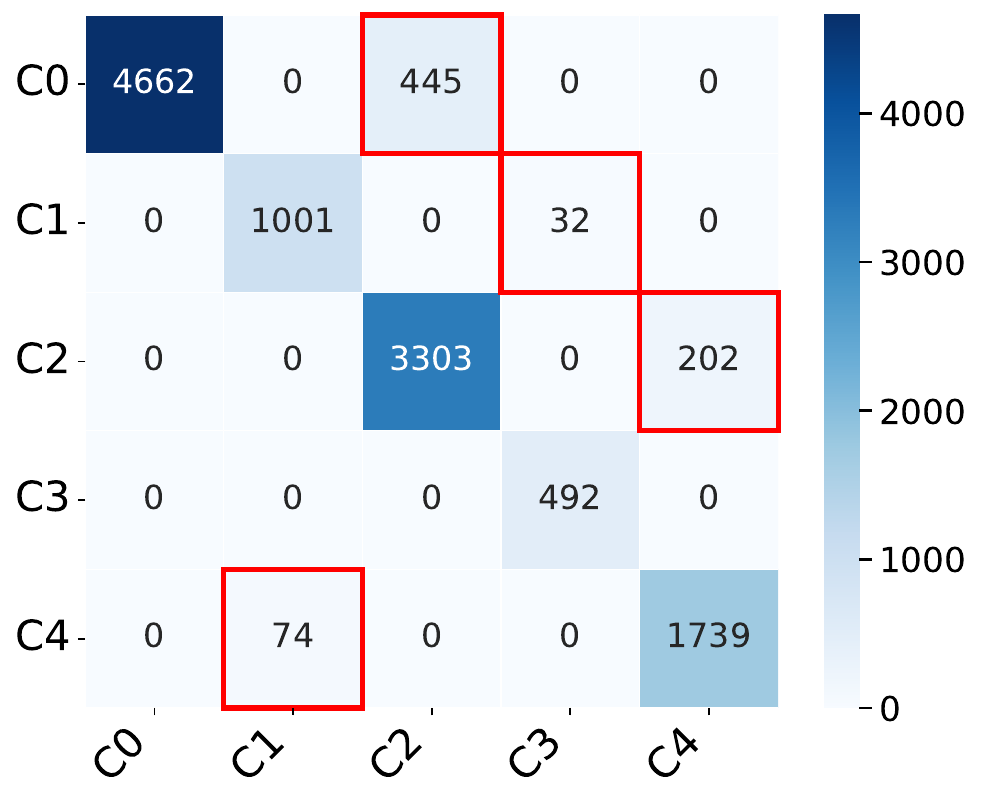}\hfill
  \includegraphics[width=0.24\textwidth]{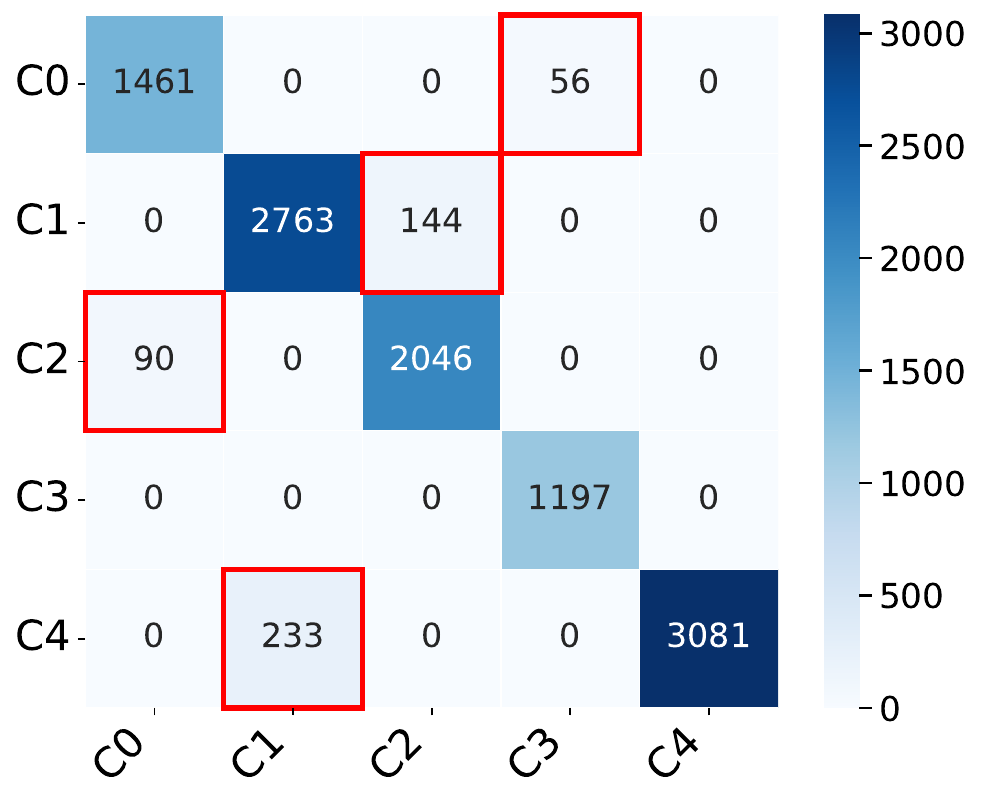}\hfill
  \includegraphics[width=0.24\textwidth]{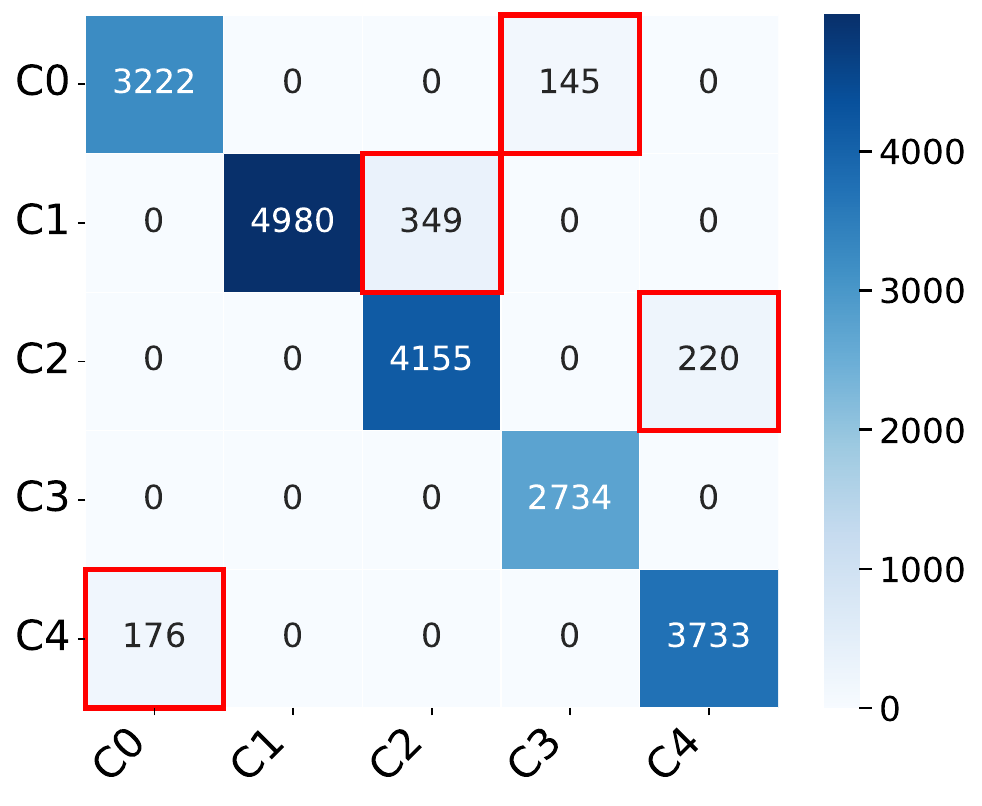}\hfill
  \includegraphics[width=0.24\textwidth]{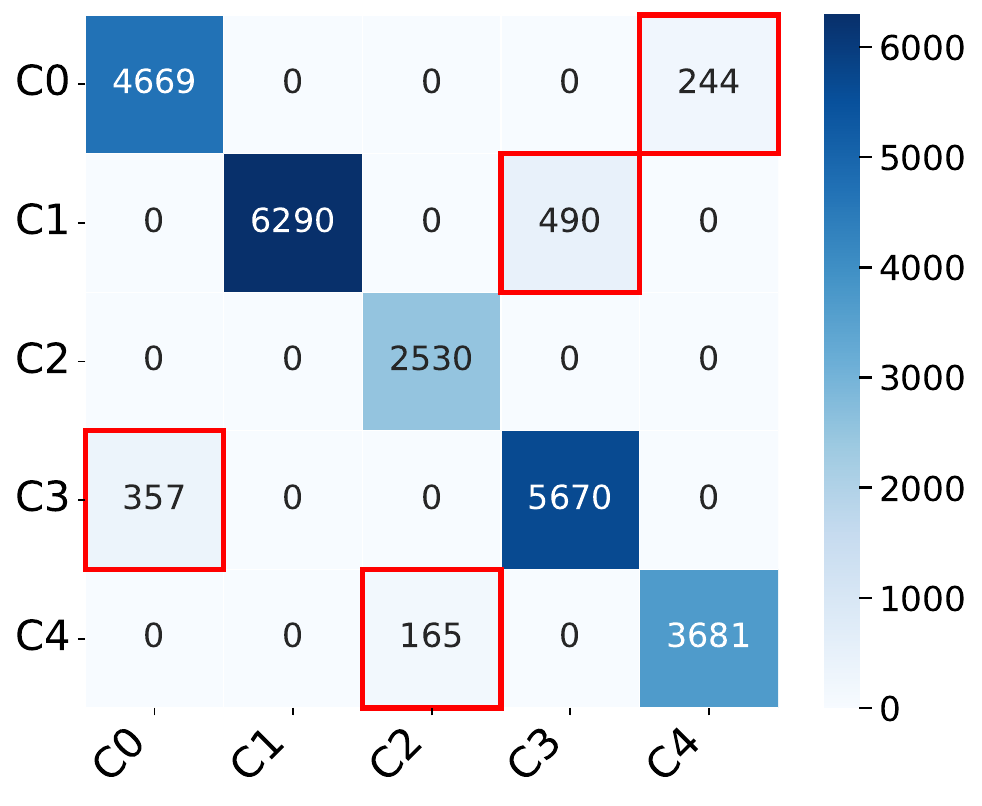}

  \vspace{0.8em}

  \includegraphics[width=0.24\textwidth]{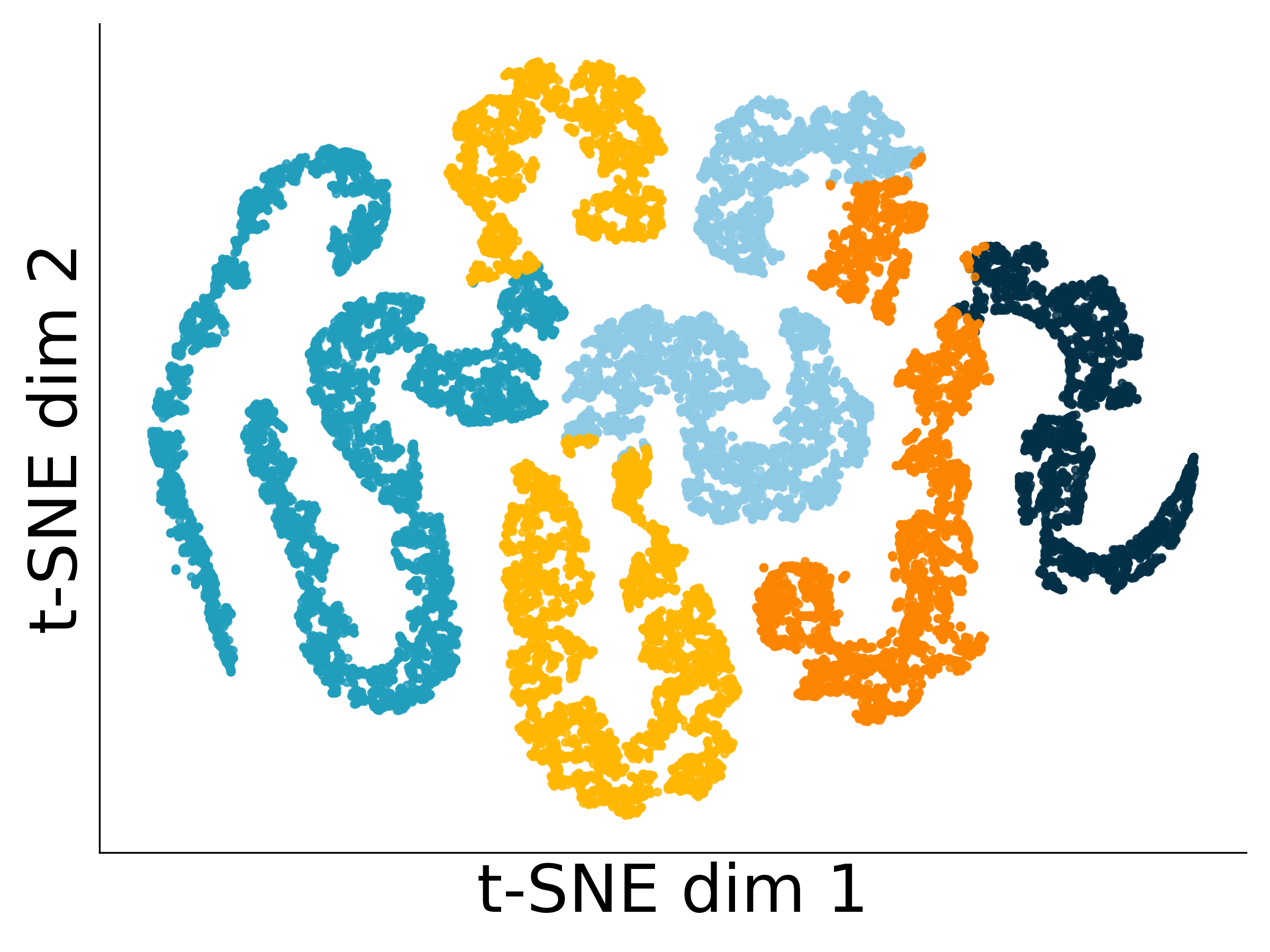}\hfill
  \includegraphics[width=0.24\textwidth]{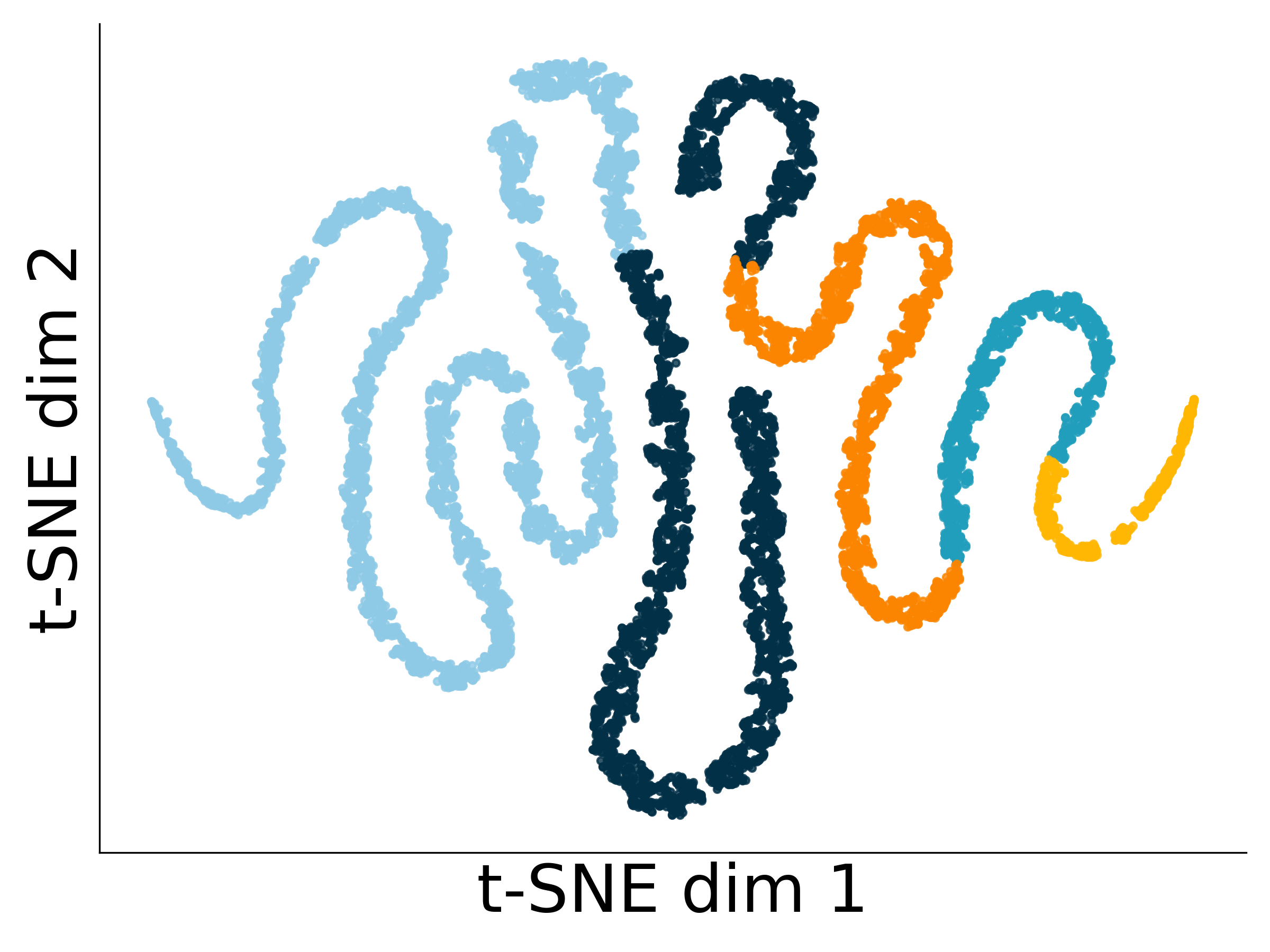}\hfill
  \includegraphics[width=0.24\textwidth]{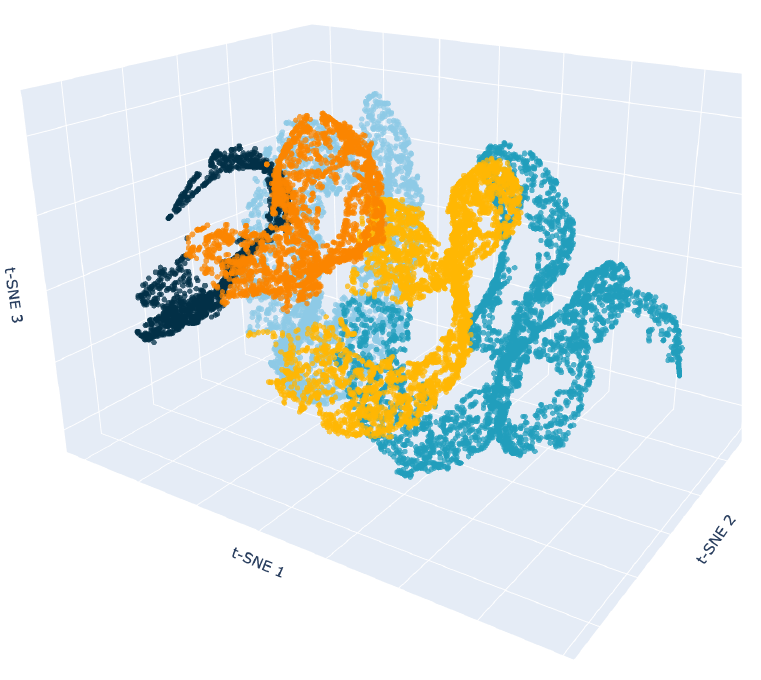}\hfill
  \includegraphics[width=0.24\textwidth]{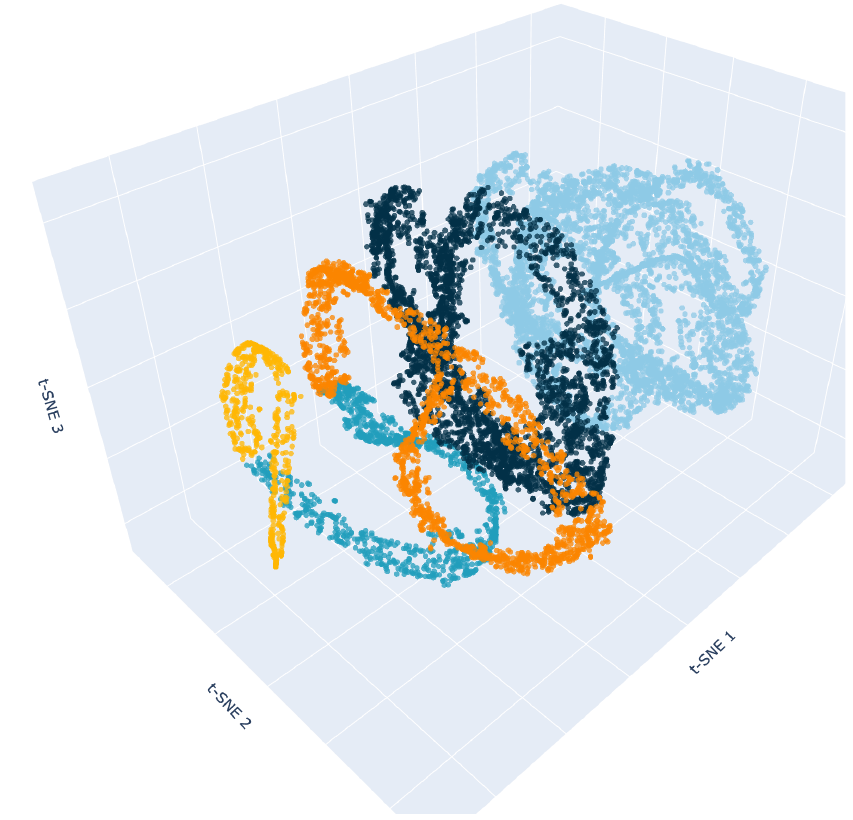}

  \caption{\textbf{Top row:} inter-step cluster transition counts (raw, unnormalized) in KMeans-ID order on \textsc{gsm8k}, \textsc{aqua}, \textsc{logiqa}, and \textsc{math} (left to right); per-panel colorbars show dataset-specific scales. Diagonal entries are \emph{stayers} (self-loops); \textcolor{red}{red} off-diagonals are \emph{crossers}. Heavy diagonals indicate that consecutive steps overwhelmingly stay within the same cluster. \textbf{Bottom row:} t-SNE projections of the same latent space --- 2D views for \textsc{math} and \textsc{gsm8k} (left two) and the corresponding 3D views (right two). For \textsc{gsm8k}, the cluster transition order is \textbf{0 $\to$ 2 $\to$ 4 $\to$ 1 $\to$ 3} according to the step-average statistics, which aligns with the t-SNE projections in both dimensionalities.}
  \label{fig:cluster-structure}
\end{figure*}

\section{Reasoning-State Transition Analysis}
\label{sec:transition-analysis}
To understand the structure of reasoning trajectories produced by LRMs 
and to motivate transition-level intervention, we analyze how these models move 
through their own reasoning process. Our analysis proceeds in three stages: 
segmentation, clustering, and characterization of the resulting trajectory 
dynamics. We segment each reasoning trace at sentence-initial discourse markers, rather than paragraph breaks, since the latter often split a coherent thought.

\paragraph{Cluster induction from mean-pool step features.}
Rather than a fixed set of human pre-defined thought
types~\cite{chen2025seal}, we embed each segmented thought as the mean of its
last-layer hidden states~\cite{yu2026explain, wu_ctrls} under a base reference model
(Qwen3-4B-Base) and apply $K$-means with $K{=}5$ to recover latent reasoning states
(details in Section~\ref{sec:method:states}). Treating the surface discourse marker as
auxiliary rather than as the category decouples segmentation from clustering granularity,
letting us steer at any resolution.

\paragraph{Incorrect traces are longer, and transitions are dominated by self-loops.}
Two observations, taken together, motivate a control mechanism targeted at self-loops.
First, incorrect traces are substantially longer than correct ones on every dataset
(Figure~\ref{fig:token-counts}), so trace length correlates with failure rather than with
additional useful reasoning. Second, across traces the inter-step transition structure is
dominated by self-loops: the transition heatmaps have heavy diagonals
(Figure~\ref{fig:cluster-structure}) and the average dwell probability---the chance a step
stays in its current state---exceeds $0.93$ on all four datasets
(Appendix~\ref{app:phase-trajectory}). Since the large majority of steps re-enter their
current state, the additional length of a long (and disproportionately incorrect) trace is
spent mostly on repeated re-entry rather than on reaching new reasoning states. This is
consistent with prior reports that excessive thought transitions correlate with degraded
reasoning quality~\citep{wang2025thoughts,fu2024efficiently}, and motivates a mechanism
that can move the trajectory out of a stalled state. We report these transition statistics
aggregated over traces.

\paragraph{A shared latent trajectory.}
Ordering clusters by mean step position reveals a consistent progression across all four
datasets (setup $\to$ exploration/calculation $\to$ consolidation). We defer the full
trajectory analysis---per-dataset cluster semantics, phase-ordered transition matrices,
and t-SNE geometry---to Appendix~\ref{app:phase-trajectory}; the method does not rely on
the phase ordering, only on the pervasiveness of self-loops documented above.

\section{Contrastive Transition Vectors}
\label{sec:justification}

The method of Section~\ref{sec:method} extracts a per-cluster steering
vector as the residual-stream contrast between crossers and stayers,
$v_c^{(\ell)} = \bar h_{c,\rightarrow}^{(\ell)} - \bar h_{c,\circlearrowright}^{(\ell)}$,
and adds a positive multiple of it whenever the model is detected to be
self-looping in cluster $c$. This is a principled local primitive for
breaking self-loops, for two reasons. First, writing $p_c(e)$ for the
probability that the next step exits $c$ under a residual offset $e$, a
first-order argument shows that any small $e$ with
$\langle e,\, \nabla_{\!e}\, p_c(0)\rangle > 0$ raises the exit
probability (Lemma~\ref{lem:exit-response}, Appendix~\ref{app:justification}).
Second, under a shared-covariance Gaussian model of crosser and stayer
activations, $v_c^{(\ell)}$ is exactly the (unwhitened) Fisher linear
discriminant separating the two classes
(Proposition~\ref{prop:fisher}, Appendix~\ref{app:justification}). The
discriminant is a \emph{correlational} proxy for the causal gradient
$\nabla_{\!e}\, p_c(0)$, not the gradient itself, so we do not assume
alignment; the random and negative controls in Section~\ref{sec:results}
($+v_c$ raises the exit rate, norm-matched random does not, and $-v_c$
suppresses it) are what verify the alignment empirically.

\begin{figure*}[t]
    \centering
    \includegraphics[width=0.98\linewidth]{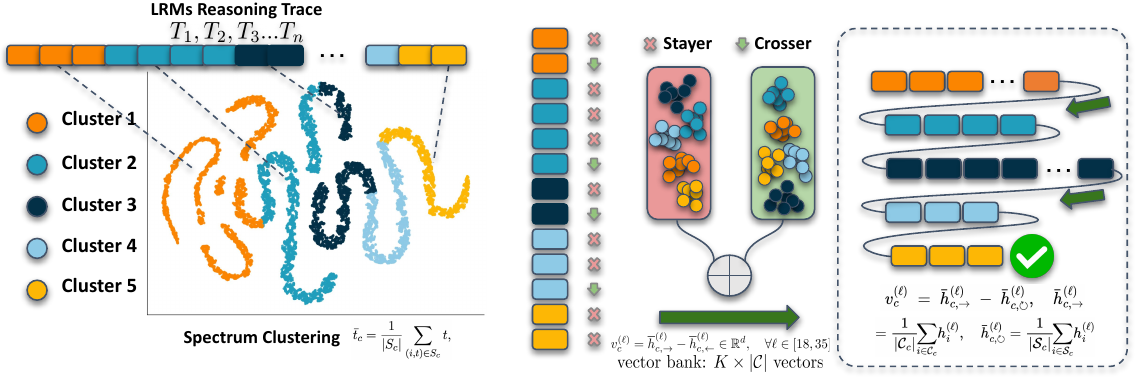}
    \caption{Pipeline for cluster-conditioned reasoning-state steering.
    Stage~I embeds each segmented step with a base reference model
    $\mathcal{M}_{\text{emb}}$ and clusters the resulting features into
    $K{=}5$ latent reasoning states. Stage~II re-aggregates cached
    residual-stream activations of the steering target
    $\mathcal{M}_{\text{tgt}}$ into one crosser-minus-stayer direction
    per (cluster, decoder layer). Stage~III is the online controller:
    after each step closes, it re-classifies the step in cluster space,
    detects self-loops, and patches the residual stream of
    $\mathcal{M}_{\text{tgt}}$ along the corresponding direction during
    the next step.
}
    \label{fig:framework}
\end{figure*}

\section{Methodology}
\label{sec:method}

Building on the analysis (Section~\ref{sec:transition-analysis}) and justification
(Section~\ref{sec:justification}), we propose a training-free framework that
operationalises self-loop intervention as cluster-conditioned activation steering
(Figure~\ref{fig:framework}): three offline stages---embed each step, cluster into
$K{=}5$ latent states, and extract per-(cluster, layer) crosser-minus-stayer
directions---and an online controller that classifies each completed step, detects
self-loops, and patches the residual stream. The subsections below detail each
component.

The pipeline runs two models: an embedding model
$\mathcal{M}_{\text{emb}}$ (Qwen3-4B-Base) that computes step features and
cluster assignments, and a steering target $\mathcal{M}_{\text{tgt}}$
(Qwen3-4B-Thinking-2507 in the main results)~\cite{yang2025qwen3} whose residual
stream is intervened on and whose generations are evaluated. Using a
non-chat-tuned $\mathcal{M}_{\text{emb}}$ gives stable cluster features that
depend on step content and do not drift as $\mathcal{M}_{\text{tgt}}$ changes.

\subsection{Latent Reasoning-State Induction}
\label{sec:method:states}

\paragraph{Mean-pool step features.}
For each segmented trace $(q, T_{1}, \ldots, T_{n})$ we run a single
forward pass of $\mathcal{M}_{\text{emb}}$ over the concatenation
$(q \,\|\, T_{1} \,\|\, \cdots \,\|\, T_{n})$ and extract the
last-layer hidden states $\{h_t^{\,\text{emb}}\}$. For step $T_{i}$
spanning tokens $[s_{i}, e_{i})$, the step feature is the token-mean
\begin{equation}
z_{i} \;=\; \frac{1}{e_{i} - s_{i}} \sum_{t=s_{i}}^{e_{i}-1} h_t^{\,\text{emb}}
\;\in\; \mathbb{R}^{d_{\text{emb}}}.
\label{eq:meanpool}
\end{equation}
We use the mean-pool feature in~\eqref{eq:meanpool} as the step
representation, so that each step is embedded by its own content.

\paragraph{Per-dataset $K$-means.}
We z-score the step features $\{z_{i}\}$ across all steps of a dataset,
fitting per-dimension $(\mu, \sigma)$ once per dataset, and run
$K$-means with $K = 5$ clusters and seed $42$ on the z-scored features.
The result is a per-dataset partition with cluster ids
$z_{i} \in \{0, \ldots, K{-}1\}$ and centroids $\{\mu_{c}\}_{c=0}^{K-1}$,
all stored in z-scored coordinates.

We fit clusters \emph{per dataset}: cluster ids do not align across datasets
(Hungarian alignment $<0.4$), so each cluster has a dataset-local id and the
steering-vector bank is fit and used per dataset.

\subsection{Per-Cluster Crosser-Minus-Stayer Vectors}
\label{sec:method:vectors}

\paragraph{Crosser and stayer pools.}
Given the cluster ids from Section~\ref{sec:method:states}, we define a
binary indicator on each non-terminal step $i$ of every trace,
\begin{equation}
\beta_{i} \;=\; \mathbf{1}\!\big[\, z_{i} \neq z_{i+1} \,\big].
\end{equation}
Steps with $\beta_{i} = 1$ are \emph{crossers} (the trajectory exits
cluster $z_{i}$ on the next step); steps with $\beta_{i} = 0$ are
\emph{stayers} (the trajectory remains in $z_{i}$). For each cluster
$c \in \{0, \ldots, K{-}1\}$ we collect all of its crossers and all of
its stayers across the dataset:
\begin{equation}
\begin{aligned}
\mathcal{C}_{c} &\;=\; \{\, i : z_{i} = c,\; \beta_{i} = 1 \,\},\\[0.2em]
\mathcal{S}_{c} &\;=\; \{\, i : z_{i} = c,\; \beta_{i} = 0 \,\}.
\end{aligned}
\end{equation}
We deliberately do \emph{not} condition the pools on trace correctness
at extraction time: per-cluster pools are already thin -- a few hundred
crossers per cluster on our 4-dataset corpus -- and conditioning on
correctness would halve them again. Trace correctness enters only at
evaluation time, where we restrict the held-out stayer prefixes to
incorrect traces (Section~\ref{sec:hit-rate-protocol}).

\paragraph{Residual-state contrast.}
Independently of the embedding pass, we record at every step $i$ the
residual-stream activation $h_{i}^{(\ell)} \in \mathbb{R}^{d}$ of the
steering target $\mathcal{M}_{\text{tgt}}$ at the discourse-marker token
that opens step $i$, for every decoder layer $\ell$ in a patched range
$\mathcal{L}$ (see hyperparameters below). The states are recorded once,
during a one-shot cache-extraction pass over the same corpus. For each
cluster $c$ and each patched layer $\ell$ we form the contrastive
direction
\begin{equation}
\begin{aligned}
v_{c}^{(\ell)}
&\;=\;
\bar h_{c,\rightarrow}^{(\ell)}
\;-\;
\bar h_{c,\circlearrowright}^{(\ell)},
\\[0.4em]
\bar h_{c,\rightarrow}^{(\ell)}
&\;=\;
\frac{1}{|\mathcal{C}_{c}|}\sum_{i \in \mathcal{C}_{c}} h_{i}^{(\ell)},
\\[0.4em]
\bar h_{c,\circlearrowright}^{(\ell)}
&\;=\;
\frac{1}{|\mathcal{S}_{c}|}\sum_{i \in \mathcal{S}_{c}} h_{i}^{(\ell)}.
\end{aligned}
\label{eq:transition-vector}
\end{equation}
By Proposition~\ref{prop:fisher}, $v_{c}^{(\ell)}$ is the unwhitened
Fisher direction separating crossers from stayers in cluster $c$ at
layer $\ell$. Under the residual-stream perturbation model of
Section~\ref{sec:justification}, a positive multiple of $v_{c}^{(\ell)}$
raises the local probability of exiting $c$ to first order \emph{to the
extent that this discriminant is aligned with the causal exit gradient}
$\nabla_{\!e}\, p_{c}(0)$; we do not assume this alignment but verify it
empirically in Section~\ref{sec:results}.

A cluster whose crosser or stayer pool is empty after filtering receives no
vector and is excluded from the controller and eval (in our corpus, the
intro-tag cluster on most datasets, plus one cluster on Gemma-4-E2B/aqua).

\subsection{Online Controller: Self-Loop Detection and Residual Patching}
\label{sec:method:online}

At inference time we run $\mathcal{M}_{\text{tgt}}$ with greedy decoding
and run $\mathcal{M}_{\text{emb}}$ alongside it as an online cluster
classifier. Both models share an architecture, so residual-stream
offsets recorded at layer $\ell$ of $\mathcal{M}_{\text{tgt}}$ are
compatible with the corresponding layer of the deployed model.

\paragraph{Step boundaries and online cluster assignment.}
We maintain a rolling text buffer of the tokens emitted by
$\mathcal{M}_{\text{tgt}}$ and apply the same sentence-initial
discourse-marker regex used during segmentation. When a marker fires we
treat the substring between the previous boundary and the new marker as
the just-completed step $T_{i}$. We pass $T_{i}$ through
$\mathcal{M}_{\text{emb}}$ -- reusing a running KV cache so only the
new tokens require a forward pass -- mean-pool the last-layer hidden
states over the step span, z-score with the dataset's saved
$(\mu, \sigma)$, and assign the cluster id by nearest centroid:
\begin{equation}
z_{i} \;=\; \mathop{\arg\min}_{c \in \{0, \ldots, K-1\}} \;
\big\| z_{i}^{\,\text{feat}} - \mu_{c} \big\|_{2}.
\label{eq:online-classify}
\end{equation}

\paragraph{Self-loop gate.}
We declare a self-loop when the most recently completed step lies in
the same cluster as the one before it: $z_{i} = z_{i-1}$. This minimal
two-step gate is what triggers the intervention. It avoids steering on
the model's first visit to a cluster (which is generally productive)
and only acts when the model has already failed to leave the cluster
once on its own; it requires only the two most recent cluster ids and
adds no learned parameters.

\paragraph{Residual-stream patching.}
When the gate fires for a cluster $c$ that has a steering vector, the
controller patches the residual stream of $\mathcal{M}_{\text{tgt}}$
across the entirety of the next decoded step. At every patched layer
$\ell \in \mathcal{L}$ and every token position $t$ inside the next
step we apply
\begin{equation}
\widetilde h^{(\ell)}_{t}
\;\leftarrow\;
h^{(\ell)}_{t}
\;+\;
\alpha \cdot \big\| h^{(\ell)}_{t} \big\|_{2} \cdot
\frac{v_{c}^{(\ell)}}{\big\| v_{c}^{(\ell)} \big\|_{2}}.
\label{eq:patch}
\end{equation}
In~\eqref{eq:patch}, per-layer unit-normalisation of $v_{c}^{(\ell)}$ scaled by
the ambient residual norm $\|h^{(\ell)}_{t}\|_{2}$ makes the offset insensitive to
the absolute residual scale (which varies with depth), so a single $\alpha$
controls strength; we patch every token of the next step at every layer in
$\mathcal{L}$. Once the step closes and is re-classified
by~\eqref{eq:online-classify}, the patch is removed and decoding continues greedily.

\paragraph{Hyperparameters.}
We fix $K = 5$ (KMeans, seed $42$) and use a per-model patched layer
range and intervention strength $\alpha$; all other choices are fixed
across datasets and models. Full settings are in
Appendix~\ref{app:hyperparams}.

\begin{table*}[t]
    \centering
    \caption{Per-cluster hit-rate (\%) on all stayer prefixes from
held-out incorrect traces (the number of prefixes therefore varies across
cells), across four reasoning datasets and three
steering targets. \sophia{} patches the next step's residual stream
with $+\alpha v_{c}^{(\ell)}$; \textbf{Negative steering} uses
$-\alpha v_{c}^{(\ell)}$; \textbf{Random steering} is a norm-matched
Gaussian; \textbf{Greedy decoding} is the unintervened baseline.
Cluster ids are dataset- and model-local; layer set $\mathcal{L}$ and
intervention strength $\alpha$ follow Section~\ref{sec:method:online}.
For Gemma-4-E2B aqua only three clusters survive intro-tag filtering.}
    \label{tab:results_per_cluster}
    \setlength{\tabcolsep}{4pt}
    \renewcommand{\arraystretch}{1.2}
\resizebox{\textwidth}{!}{%
    \begin{tabular}{l cccc cccc cccc cccc}
    \toprule
     & \multicolumn{4}{c}{gsm8k} & \multicolumn{4}{c}{aqua} & \multicolumn{4}{c}{logiqa} & \multicolumn{4}{c}{math} \\
    \cmidrule(lr){2-5} \cmidrule(lr){6-9} \cmidrule(lr){10-13} \cmidrule(lr){14-17}
    Method & C0 & C1 & C2 & C3 & C0 & C1 & C2 & C3 & C0 & C1 & C2 & C3 & C0 & C1 & C2 & C3 \\
    \midrule
    \multicolumn{17}{l}{\textit{Qwen3-32B-Thinking}} \\
    \midrule
    Greedy decoding   & \textbf{28.6} & 89.7 & \underline{20.0} & 42.9 & 53.3 & \textbf{34.5} & 17.9 & \underline{44.8} & \textbf{16.7} & \textbf{11.5} & \underline{81.8} & 0.0 & 10.3 & \underline{75.9} & 73.3 & \underline{29.6} \\
    Random steering   & 20.7 & 82.8 & 17.2 & 40.7 & 48.3 & \textbf{34.5} & 17.9 & 39.3 & \textbf{16.7} & \underline{7.7} & 79.2 & 0.0 & 10.0 & 73.3 & 75.9 & 21.4 \\
    Negative steering & \underline{26.7} & \underline{96.7} & 17.2 & \textbf{70.0} & \underline{83.3} & 27.6 & \underline{20.7} & 25.0 & 3.3 & 4.0 & 57.9 & 0.0 & \textbf{20.7} & 64.3 & \textbf{86.2} & 20.0 \\
    \rowcolor{sealblue}
    SOPHIA            & \textbf{28.6} & \textbf{100.0} & \textbf{23.3} & \underline{67.9} & \textbf{92.9} & \underline{31.0} & \textbf{21.4} & \textbf{70.8} & \underline{13.3} & \underline{7.7} & \textbf{85.2} & 0.0 & \underline{16.7} & \textbf{82.8} & \underline{81.2} & \textbf{39.3} \\
    \midrule
    \multicolumn{17}{l}{\textit{Gemma-4-E2B}} \\
    \midrule
    Greedy decoding   & \textbf{98.3} & 81.2 & 0.0 & 32.0 & \underline{33.3} & \underline{5.0} & 33.3 & \underline{34.8} & \underline{40.0} & 9.1 & 0.0 & 6.9 & \textbf{40.9} & 4.8 & 10.0 & \underline{20.0} \\
    Random steering   & 94.2 & \underline{83.6} & \textbf{21.4} & \underline{35.7} & \textbf{39.3} & 4.5 & \textbf{50.0} & 30.2 & 36.7 & \textbf{15.4} & 0.0 & \textbf{10.7} & 19.0 & 3.8 & \underline{20.0} & 12.5 \\
    Negative steering & 86.7 & 71.5 & 5.0 & 24.0 & 26.9 & 4.3 & 16.7 & \textbf{42.5} & 20.0 & \textbf{15.4} & 0.0 & 9.1 & 18.5 & \textbf{10.3} & 5.3 & 18.5 \\
    \rowcolor{sealblue}
    SOPHIA            & \underline{97.4} & \textbf{85.1} & \underline{15.4} & \textbf{41.7} & 30.8 & \textbf{29.2} & \underline{37.5} & 32.9 & \textbf{40.7} & \underline{13.6} & 0.0 & \underline{10.3} & \underline{34.8} & \underline{5.3} & \textbf{30.8} & \textbf{26.3} \\
    \midrule
    \multicolumn{17}{l}{\textit{Qwen3-4B-Thinking-2507}} \\
    \midrule
    Greedy decoding   & \underline{7.1} & 26.7 & 17.9 & 25.0 & 0.0 & \underline{66.7} & 27.6 & \textbf{23.3} & 3.4 & \underline{31.0} & 23.3 & 50.0 & 17.4 & \underline{50.0} & \underline{11.1} & \underline{11.1} \\
    Random steering   & 6.7 & \underline{33.3} & 10.0 & 62.5 & \underline{3.4} & 63.3 & \underline{34.5} & \underline{20.7} & \underline{6.7} & \underline{31.0} & \underline{26.7} & \underline{60.0} & 22.2 & \textbf{63.0} & 3.7 & 3.8 \\
    Negative steering & 3.7 & 10.0 & \underline{20.0} & \underline{75.0} & 0.0 & 46.7 & 0.0 & 6.7 & 0.0 & 23.1 & 20.0 & 34.5 & \underline{33.3} & 34.5 & 0.0 & 5.6 \\
    \rowcolor{sealblue}
    SOPHIA            & \textbf{20.7} & \textbf{43.3} & \textbf{62.1} & \textbf{100.0} & \textbf{51.7} & \textbf{88.5} & \textbf{76.7} & \textbf{23.3} & \textbf{30.0} & \textbf{80.0} & \textbf{37.9} & \textbf{65.5} & \textbf{70.0} & 48.0 & \textbf{48.0} & \textbf{17.9} \\
    \bottomrule
    \end{tabular}%
    }
\end{table*}

\section{Experiments}
\label{sec:results}

The trajectory analysis in Section~\ref{sec:transition-analysis}
identified self-loops as a dominant failure mode of LRM reasoning, and
Section~\ref{sec:method} extracted, for each latent cluster, a
crosser-minus-stayer steering direction predicted by
Section~\ref{sec:justification} to be a usable local exit primitive. We
now ask whether that direction has the predicted causal effect: can a
small additive offset along $v_{c}^{(\ell)}$ actually push the model out
of cluster $c$ on the next step, when the model would otherwise have
self-looped?

\subsection{Evaluation Protocol: Per-Cluster Hit Rate}
\label{sec:hit-rate-protocol}

To test whether $v_{c}$ breaks self-loops, we evaluate on all stayer
prefixes of each cluster in the held-out incorrect traces, so the number
of prefixes varies across clusters. A stayer prefix is
a step $i$ in an incorrect trace with $z_{i} = z_{i-1}$ (the model has
already failed to leave $c$ once on its own) and $z_{i} = c$. For each
prefix we replay the model under four conditions, decode the next
reasoning step, and re-classify the produced step into a cluster using
$\mathcal{M}_{\text{emb}}$ (Section~\ref{sec:method:online}). The
per-cluster \emph{hit rate} (equivalently, the \emph{loop-exit rate}) is the
fraction of stayer prefixes for which the produced next step is assigned to a
cluster other than $c$. This metric is intentionally local: it asks only whether the
intervention moved the trajectory out of $c$ on the very next step,
decoupling the test of the steering direction itself from any
downstream effect on final-answer accuracy or token count. It certifies that the model
\emph{left} the looped state; it does \emph{not} certify that the resulting step is more
productive, more coherent, or that the final answer improves---those are end-to-end
questions we deliberately do not settle here.

The four conditions are: \textbf{Greedy decoding} (no intervention; the
natural rate at which the model spontaneously escapes $c$ from a
two-step stayer prefix); \textbf{Random steering} (a norm-matched
Gaussian vector, deterministic per (cluster, dataset, sample\_id,
step\_idx), which controls for the effect of perturbing the residual
stream at all); \textbf{Negative steering} ($-\alpha\, v_{c}^{(\ell)}$
applied via Equation~\ref{eq:patch}, which should keep the model in
$c$ if $v_{c}^{(\ell)}$ is sign-meaningful); and \sophia{}
($+\alpha\, v_{c}^{(\ell)}$ via Equation~\ref{eq:patch}). Cluster ids
are dataset-local and model-local: Hungarian alignment of clusters
across (model, dataset) pairs fell below $0.4$, so we re-fit the latent
state structure for each pair rather than impose a shared identity on
clusters that are not in fact aligned across models.

\subsection{Results}
\label{sec:hit-rate-results}

Table~\ref{tab:results_per_cluster} reports per-cluster hit rates across
three models and four reasoning datasets. Several patterns are robust
across the grid.

\paragraph{\sophia{} breaks self-loops where greedy decoding cannot.}
Of the cells where \sophia{} is applicable, it attains the highest hit
rate in the majority. The improvement over Greedy decoding is largest
where Greedy is weakest -- that is, on stayer prefixes the model would
otherwise have remained stuck on. On Qwen3-4B-Thinking-2507---our main
steering target---the gains are large and consistent: gsm8k C3 rises from
$25.0\%$ to $100\%$, gsm8k C2 from $17.9\%$ to $62.1\%$, aqua C0 from
$0.0\%$ to $51.7\%$, and logiqa C2 from $31.0\%$ to $80.0\%$, with
\sophia{} the best or tied-best method in $15$ of the $16$ applicable cells.
On the larger Qwen3-32B-Thinking and the different-family Gemma-4-E2B the
picture is more mixed: \sophia{} still delivers the single largest gains on
several cells (e.g.\ Qwen3-32B aqua C0: $53.3\%\to92.9\%$; Gemma aqua C1:
$5.0\%\to29.2\%$, math C2: $10.0\%\to30.8\%$), but it is at or below Greedy
on a non-trivial minority of cells, so transfer to these targets is partial
rather than uniform. The extraction procedure is held fixed across all three
targets---only the scalar strength $\alpha$ is retuned per model---so where
the gains do appear they are direct causal evidence that the contrastive
direction $v_{c}$ encodes a usable ``leave-$c$'' signal that survives a change
of model size and family.

\paragraph{The effect is direction-specific, not magnitude-driven.}
Random steering at the same residual norm produces hit rates that are
typically indistinguishable from Greedy decoding, confirming that
arbitrary residual-stream perturbations do not by themselves move the
trajectory out of $c$. Negative steering, which uses the same vector
with reversed sign, generally lowers the hit rate further (e.g.\ on
Qwen3-4B aqua C2 it falls from $27.6\%$ to $0.0\%$, and on aqua C3 from
$23.3\%$ to $6.7\%$), as Lemma~\ref{lem:exit-response} predicts if
$v_{c}$ is positively aligned with $\nabla_{e}\, p_{c}(0)$. A handful of
cells -- e.g.\ Qwen3-4B gsm8k C3 -- show negative steering also
improving over greedy; we read these as evidence that some clusters are
unstable equilibria where any sufficiently large perturbation forces an
exit. Together, the random and negative controls rule out the
hypothesis that the \sophia{} effect is a generic decoding-perturbation
artefact.

\paragraph{The corrective direction is state-specific.}
The loop-exit gain from \sophia{} over Greedy varies sharply across
clusters \emph{within a single (model, dataset) pair}. On Qwen3-4B
logiqa, for instance, \sophia{} produces a $+49$pt gain on C1
($31.0\%\to80.0\%$) but only a $+15$pt gain on C2 ($23.3\%\to37.9\%$);
on math it gains $+37$pts on C2 but slightly trails Greedy on C1. If a
single global ``break-the-loop'' direction sufficed, we would expect a
roughly uniform shift across clusters within a dataset. The
cluster-specific spread is instead what we would expect if the
corrective signals for different reasoning states live in different
residual-stream directions -- the hypothesis raised in
Section~\ref{sec:transition-analysis}. We caution that spread alone is
only suggestive: a state that is intrinsically easier to exit could
also show a larger gain.

\section{Related Work}
\label{app:related}

\paragraph{Chain-of-thought reasoning.}
CoT prompting and related test-time reasoning methods improve model performance by making intermediate computation visible in generated text \cite{wei2022chain,kojima2022large}.
Sampling and search methods such as self-consistency and tree-based reasoning improve reliability by exploring multiple possible traces \cite{wang2023selfconsistency,yao2023tree}.
Process supervision and verifier-guided methods further evaluate intermediate reasoning steps rather than only final answers \cite{lightman2023lets}.
\method{} is complementary and orthogonal to these: it does not branch over many traces or train a verifier, but intervenes inside a single decoding trajectory to bias the next latent-state transition. In particular, search- and verifier-guided methods select among or score \emph{whole} candidate steps or traces, whereas \method{} acts on the residual stream to change which step is generated in the first place.

\paragraph{Latent representations of reasoning.}
Several studies suggest that transformer hidden states encode task-relevant intermediate information, including entities, partial answers, and abstract reasoning features \cite{geva2021transformer,hernandez2023linearity}.
Closer to our setting, prior work abstracts CoT steps into latent reasoning states and models their transitions \cite{wu_ctrls,yu2026explain}, quantifies step-level reasoning effort along such state trajectories \cite{wei2026hard}, and tests whether latent CoT actually uses its intermediate states \cite{qiu2026latent}; we adopt this state view but use it as the unit of intervention rather than only of analysis.
Our abstraction is deliberately simple: cluster prefix activations into discrete states and analyze the induced transition graph.
This keeps the state space easy to log, visualize, and connect to intervention outcomes, and---unlike hand-defined thought taxonomies \cite{chen2025seal}---does not commit to a fixed catalogue of reasoning types.

\paragraph{Activation steering and representation engineering.}
Activation steering changes model behavior by adding a vector to a hidden representation at inference time \cite{turner2023activation,li2023inference,zou2023representation}.
Such directions can also represent and compose procedural skills \cite{jiang-etal-2026-toward}, and inference-time adaptation has been extended to LLM agents through an explicit decision layer updated online from action-level feedback \cite{yu-etal-2026-olivia}.
Most prior steering directions encode a single global attribute or behavior that is applied uniformly across generation.
\method{} instead indexes the intervention by the model's current (source) reasoning state and applies it only when an online self-loop is detected.
This turns an additive residual intervention into a local, state-conditioned exit primitive rather than a response-wide semantic bias. We do \emph{not} claim target-addressable (source-to-target) control: the vectors are conditioned on the source state and trained to induce exit from it.

\paragraph{Efficient and controlled reasoning.}
Recent work on efficient reasoning studies when extended traces help, when they waste tokens, and how to avoid unproductive thinking \cite{chen2024not,fu2024efficiently,wang2025thoughts}.
Hidden-state spectral signals have been used to decide when to stop reasoning \cite{cheng-etal-2026-stars}, group-relative RL objectives have been designed to curb overthinking during training \cite{mundada2026wsgrpo}, and decoding-level methods such as speculative decoding reduce generation cost \cite{lu2024improving}.
Our contribution is to express these issues in the language of latent-state transitions: the same trace length can correspond to useful exploration or to an unproductive loop, and state traces let us target the loop directly rather than penalizing surface markers or truncating length.
Token- or logit-level methods such as thought-switching penalties \cite{wang2025thoughts} operate on the surface token stream; \method{} operates on the latent state that those tokens realize.

\section{Conclusion}
We reframed inference-time control of reasoning as steering at the level of
latent reasoning states, with self-loops as a prominent failure signature of
overlong, incorrect traces. \sophia{} realizes this view training-free,
pairing a bank of source-state-conditioned residual-stream exit directions with
an online controller that detects self-loops and intervenes before the next
step. Across three models and four benchmarks, adding a state's exit direction
causally raises its loop-exit rate over greedy decoding; the effect is
direction-specific (norm-matched noise does not reproduce it, sign-flipping
suppresses it) and concentrated where greedy decoding cannot escape. Several
questions remain open: whether state-specific directions are strictly necessary
rather than a single global exit vector, the downstream effect on end-task
accuracy, extending from selective exit to target-addressable state-to-state
control, strengthening the controller, and aligning states across datasets.

\section*{Limitations}

\paragraph{The metric is local exit, not end-task quality.}
Our headline evidence is a single-step, per-state \emph{loop-exit rate}: given a stayer
prefix, does the next step leave the current state? This certifies that the model left a
looped state; it does not certify that the resulting step is more coherent, more
productive, or that final-answer accuracy or token count improves. We deliberately do not
report end-to-end accuracy or token-efficiency numbers, and the paper should not be read
as claiming end-task gains. Measuring those is the most important next step.

\paragraph{Control is undirected.}
The steering vectors are conditioned on the \emph{source} state and built to induce exit
from it (target event $\{z_{t+1}\neq i\}$). We therefore demonstrate the ability to push
the model \emph{out of} its current state, not the ability to route it to a \emph{chosen}
target state. Claims of target-addressable, source-to-target, or ``state-pair'' control
would require a directed-steering study, which we have not run.

\paragraph{The specificity claim is only suggestive.}
Our argument that source-state indexing is necessary rests on the cluster-to-cluster
spread of exit gains. Spread alone is only suggestive---an intrinsically easier-to-exit
state could show a larger gain---and we do not include a direct comparison against a
single pooled ``global'' exit direction; that comparison is left to future work.

\paragraph{Statistical power.}
Each per-cluster hit rate is estimated from all held-out stayer prefixes of that cluster,
so the sample size differs across cells; we do not
report confidence intervals or significance tests, and small differences between
conditions (a few points) are within binomial noise. Transfer is uneven: gains are large
and consistent on Qwen3-4B-Thinking-2507 but partial on Qwen3-32B-Thinking and
Gemma-4-E2B, where \sophia{} ties or trails greedy on a minority of cells.

\paragraph{The self-loop ``failure mode'' is evidenced indirectly.}
We motivate self-loops as a failure signature via (i) the correlation between trace length
and incorrectness and (ii) the heavy transition-matrix diagonals, but our aggregate
transition statistics are not split by trace correctness, so we do not quantify at the
state level how much \emph{more} incorrect traces self-loop than correct ones. The
phase-ordered ``near-deterministic forward chain'' is also partly a consequence of ordering
states by mean step position rather than a purely emergent property.

\paragraph{Design dependencies.}
The pipeline depends on a separate base model as an online state classifier, on
per-dataset $K$-means with $K{=}5$ (clusters are not aligned across datasets, so vectors
are fit and used per dataset), and on a per-model intervention strength $\alpha$. The gate
fires on repeated self-loops but does not prevent the intervention from pushing a trace
\emph{into} the absorbing rumination phase $P_5$; a dwell-aware gate that accounts for the
target phase is left to future work.

\paragraph{LLM usage disclosure.}
AI tools were used for minor writing, formatting, and visual editing assistance, including
help generating illustrative figures. All claims, experiments, and analyses are the
authors' own.

\bibliography{main}

\clearpage
\appendix
\section{Appendix}

This appendix states the general control objective that motivates the method, while the main text keeps the presentation operational.
Let $q_\phi(z\mid \xi_t)$ be the latent-state abstraction (a general map from prefix to state; in our implementation it is \emph{non-parametric}, with $\phi$ denoting the $K$-means centroids rather than trained weights) and let $P^0(\cdot\mid i,\xi_t)$ be the next-state distribution induced by the base model from source state $i$ and prefix $\xi_t$.
An activation intervention $u_t=(\alpha_t,v_t)$ changes the residual stream and induces a controlled transition distribution $P^{u_t}(\cdot\mid i,\xi_t)$.

\begin{definition}[Transition steering vector]
For source state $i$, target state $j$, and context $c_t$, a vector $v_{i\to j}^{(\ell)}(c_t)$ is a transition steering vector if there exists a positive strength $\alpha$ such that
\begin{equation}
    P^{(\alpha,v_{i\to j}^{(\ell)}(c_t))}(j\mid i,\xi_t)
    >
    P^{0}(j\mid i,\xi_t).
\end{equation}
\end{definition}

\paragraph{Evaluated special case: state exit.}
The definition above is stated for a specific target state $j$, but this paper evaluates only the \emph{exit} special case, in which the target is the event $\{z_{t+1}\neq i\}$---leaving the current state, regardless of destination. The per-source-state vector $v_c^{(\ell)}$ of Section~\ref{sec:method:vectors} instantiates $v_{i\to\{\neq i\}}$: it is built to raise $\Pr[z_{t+1}\neq i\mid z_t=i]$, not the probability of any particular successor $j$.

The online controller can be written as a policy $\kappa$ that maps the current prefix, inferred state, and context to an intervention.
Given a desired transition policy $\pi^\star(\cdot\mid i,\xi_t)$, the ideal objective is
\begin{equation}
\label{eq:control-objective}
    \min_{\kappa}
    \mathbb{E}_{x,t}
    \left[
    \mathrm{KL}\left(
    \pi^\star(\cdot\mid i_t,\xi_t)
    \middle\|
    P^{\kappa}(\cdot\mid i_t,\xi_t)
    \right)
    +
    \lambda \|e_t\|_2^2
    \right],
\end{equation}
where $e_t=\alpha_t v_t$, $i_t=\arg\max_i q_\phi(i\mid \xi_t)$, and $\lambda>0$ is a regularization weight penalizing intervention magnitude (in the implementation, magnitude is instead fixed by the single strength $\alpha$ of Section~\ref{sec:method:online}).
Equation~\ref{eq:control-objective} is not solved directly.
The implemented method approximates it with a finite vector bank, a data-estimated target policy, and a gate that suppresses unsupported interventions.
The implementation in Section~\ref{sec:method} realizes the latent-state model $q_\phi$ as a per-dataset $K$-means partition over mean-pool step features, and the controller $\kappa$ as the two-step self-loop gate of Section~\ref{sec:method:online}; the steering vectors $\{v_{c}^{(\ell)}\}$ of Section~\ref{sec:method:vectors} are the finite vector bank that instantiates the \emph{exit} vector $v_{i\to\{\neq i\}}^{(\ell)}$, i.e.\ the special case of $v_{i\to j}^{(\ell)}(c_t)$ in~Eq.~\ref{eq:control-objective} with target event $\{z_{t+1}\neq i\}$.

\section{Formal Justification of the Contrastive Direction}
\label{app:justification}

This appendix gives the formal statements summarised in
Section~\ref{sec:justification}: why a positive multiple of the
crosser-minus-stayer direction is a principled local primitive for
raising the next-step probability of exiting a cluster.

\paragraph{Setup.}
Fix a decoder layer $\ell$, a step $i$ whose prefix $\xi_i$ is assigned to
cluster $z_i = c$, and the residual-stream activation
$h := h_i^{(\ell)} \in \mathbb{R}^{d}$ at the discourse-marker token that
opens step $i$. Let
\begin{equation}
p_c(e) \;:=\; \Pr\!\big[\, z_{i+1} \neq c \;\big|\; z_i = c,\;\; h \mapsto h + e \,\big]
\end{equation}
denote the probability that the next step exits cluster $c$ when the
residual at layer $\ell$ is offset by $e \in \mathbb{R}^{d}$, with
$p_c(0)$ the unperturbed exit probability. The goal of the intervention
is a scriptsize offset $e$ that increases $p_c(e)$ relative to $p_c(0)$.

\begin{lemma}[Local exit response]
\label{lem:exit-response}
If $p_c$ is differentiable at $e = 0$, then
\begin{equation}
p_c(e) \;=\; p_c(0) \;+\; \nabla_{\!e}\, p_c(0)^{\top} e \;+\; o\!\big(\|e\|_{2}\big).
\end{equation}
In particular, any sufficiently scriptsize $e$ with
$\langle e,\, \nabla_{\!e}\, p_c(0)\rangle > 0$ raises the exit
probability to first order.
\end{lemma}

\begin{proof}
First-order Taylor expansion of $p_c$ at the origin.
\end{proof}

The gradient $\nabla_{\!e}\, p_c(0)$ is the ideal local control direction, but
it is not directly accessible: $p_c$ is induced jointly by the autoregressive
generation policy of the steering target, the cluster classifier on the next
step, and the discourse-marker segmentation rule, none of which we
differentiate through end-to-end. We therefore use a tractable proxy that only
requires \emph{recorded} residual states from past traces.

\begin{proposition}[Contrastive direction as a Fisher discriminant]
\label{prop:fisher}
Conditional on $z_i = c$, suppose the residual states at layer $\ell$ of
crossers ($z_{i+1} \neq c$) and stayers ($z_{i+1} = c$) are approximately
Gaussian with class means $\mu_c^{\rightarrow}$,
$\mu_c^{\circlearrowright}$ and shared covariance $\Sigma_c$. Then the
linear discriminant maximising the Fisher separation between crossers
and stayers is
\begin{equation}
v_c^{\,\mathrm{Fisher}} \;\propto\; \Sigma_c^{-1}\!\big(\mu_c^{\rightarrow} - \mu_c^{\circlearrowright}\big),
\label{eq:fisher}
\end{equation}
and in the special case $\Sigma_c \propto I$ this reduces to the
unwhitened mean difference
$\mu_c^{\rightarrow} - \mu_c^{\circlearrowright}$ up to a positive scale.
\end{proposition}

\begin{proof}
Under shared-covariance Gaussian class-conditionals the log-likelihood
ratio between the two classes is affine in the activation vector with
linear coefficient $\Sigma_c^{-1}(\mu_c^{\rightarrow} - \mu_c^{\circlearrowright})$;
the same direction maximises the Rayleigh quotient of between-class
mean separation to within-class variance.
\end{proof}

\begin{table*}[ht]
  \centering
  \scriptsize
  \begin{tabular}{cclrl}
    \toprule
    Cluster & Phase & Avg.\ step & $n$ & Reading \\
    \midrule
    0 & $P_1$ &  6.0 & 5{,}162 & Problem reading and decomposition \\
    2 & $P_2$ & 16.8 & 3{,}748 & First sanity checks on early arithmetic \\
    4 & $P_3$ & 31.5 & 1{,}941 & Re-interpretation; partial answers \\
    1 & $P_4$ & 52.4 & 1{,}075 & Final-answer scaffolds; repeated checks \\
    3 & $P_5$ & 83.2 &   524   & Extreme re-examination; failure regime \\
    \bottomrule
  \end{tabular}
  \caption{\textsc{gsm8k}: cluster $\to$ phase mapping with mean step index
  and one-line semantics.}
  \label{tab:cluster_gsm8k}
\end{table*}

\begin{table*}[ht]
  \centering
  \scriptsize
  \begin{tabular}{cclrl}
    \toprule
    Cluster & Phase & Avg.\ step & $n$ & Reading \\
    \midrule
    4 & $P_1$ &  7.5 & 3{,}335 & Problem read; immediate second-guessing \\
    1 & $P_2$ & 22.2 & 2{,}996 & Comparing intermediate results to options \\
    2 & $P_3$ & 41.1 & 2{,}190 & Restarting derivations \\
    0 & $P_4$ & 63.1 & 1{,}551 & Failing to match a multiple-choice option \\
    3 & $P_5$ & 89.0 & 1{,}253 & Exhaustion; cyclic re-examination \\
    \bottomrule
  \end{tabular}
  \caption{\textsc{aqua}: cluster $\to$ phase mapping with mean step index
  and one-line semantics.}
  \label{tab:cluster_aqua}
\end{table*}

\begin{table*}[ht]
  \centering
  \scriptsize
  \begin{tabular}{cclrl}
    \toprule
    Cluster & Phase & Avg.\ step & $n$ & Reading \\
    \midrule
    1 & $P_1$ &  7.3 & 5{,}480 & Scanning logical conditions \\
    2 & $P_2$ & 23.7 & 4{,}504 & First attempt at case analysis \\
    4 & $P_3$ & 43.9 & 3{,}953 & Early commitments and case rejections \\
    0 & $P_4$ & 66.1 & 3{,}398 & Long case-by-case eliminations \\
    3 & $P_5$ & 90.2 & 2{,}879 & Deep nested case analysis; worst tail \\
    \bottomrule
  \end{tabular}
  \caption{\textsc{logiqa}: cluster $\to$ phase mapping with mean step index
  and one-line semantics.}
  \label{tab:cluster_logiqa}
\end{table*}

\begin{table*}[ht]
  \centering
  \scriptsize
  \begin{tabular}{cclrl}
    \toprule
    Cluster & Phase & Avg.\ step & $n$ & Reading \\
    \midrule
    1 & $P_1$ &  7.8 & 6{,}790 & Formula recall and equation setup \\
    3 & $P_2$ & 22.0 & 6{,}160 & Execution of chosen approach; first checks \\
    0 & $P_3$ & 38.4 & 5{,}026 & Algebraic re-derivations \\
    4 & $P_4$ & 58.0 & 3{,}925 & Alternative geometric/algebraic framings \\
    2 & $P_5$ & 80.3 & 2{,}695 & Minimal-content stalling; longest tails \\
    \bottomrule
  \end{tabular}
  \caption{\textsc{math}: cluster $\to$ phase mapping with mean step index
  and one-line semantics.}
  \label{tab:cluster_math}
\end{table*}

\section{Detailed Phase-Trajectory Analysis}
\label{app:phase-trajectory}

This appendix expands the trajectory analysis summarised in
Section~\ref{sec:transition-analysis}. We document, for each of the four
reasoning datasets used to build the cluster bank, the per-dataset
cluster$\to$phase mapping, the raw and phase-ordered transition matrices,
the t-SNE geometry of the step embeddings, and the aggregate
dwell/forward/backward statistics that motivate the self-loop gating
used by the method.

\paragraph{Setup.}
We embed every reasoning step using the mean-pool last-layer features of
the embedding model $\mathcal{M}_{\text{emb}}$ described in
Section~\ref{sec:method:states}, and partition the resulting
step-level representations with KMeans ($K{=}5$) on each of four datasets:
\textsc{gsm8k}, \textsc{aqua}, \textsc{logiqa}, and \textsc{math}.
The cluster identifiers returned by KMeans are arbitrary, so they cannot be
read directly across datasets; we instead annotate each cluster with its
empirical \emph{average step position} and a short semantic reading distilled
from 10 random member steps.

\paragraph{Cluster semantics.}
Tables~\ref{tab:cluster_gsm8k}--\ref{tab:cluster_math} list, for each
dataset, the five clusters together with their mean step index, member
count, and the dominant behaviour observed in their members. A consistent
five-stage pattern emerges across datasets: \emph{problem reading} $\to$
\emph{first computation} $\to$ \emph{mid verification} $\to$ \emph{late
deliberation} $\to$ \emph{rumination tail}. We adopt this as a canonical
phase labelling $P_{1}{\to}P_{5}$ in what follows; the \textsc{Phase} column
of each table records the mapping between KMeans IDs and these phase labels.

\paragraph{Raw transition matrices.}
We compute the empirical inter-step transition counts
between KMeans clusters and visualise them in
Fig.~\ref{fig:trajectory_heatmap_raw_app}. The same heatmaps appear in
Fig.~\ref{fig:cluster-structure} of the main text as part of the
two-panel transition / t-SNE figure; we reproduce them here in a
single-row layout to make the per-dataset comparison easier to read.
Plotted in raw cluster-ID order the matrices appear scrambled, because the
ID-to-phase permutation differs across datasets
(Tabs.~\ref{tab:cluster_gsm8k}--\ref{tab:cluster_math}),
so the same underlying dynamics produce different cell patterns.

\begin{figure}[ht]
  \centering
  \includegraphics[width=0.24\linewidth]{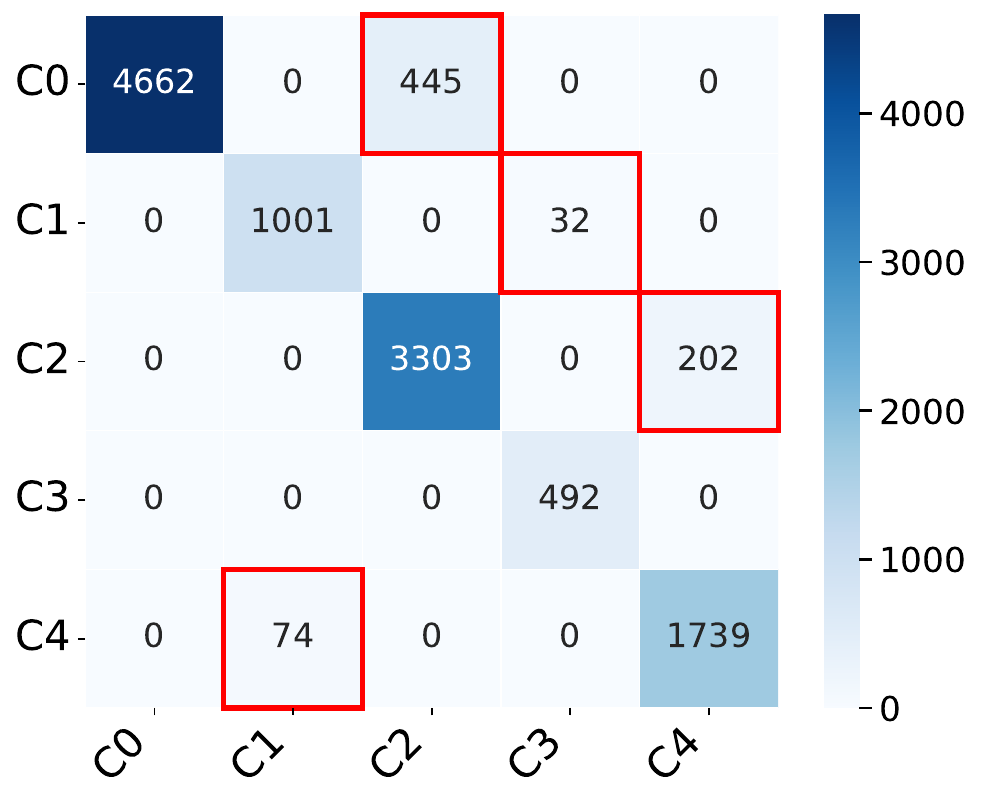}\hfill
  \includegraphics[width=0.24\linewidth]{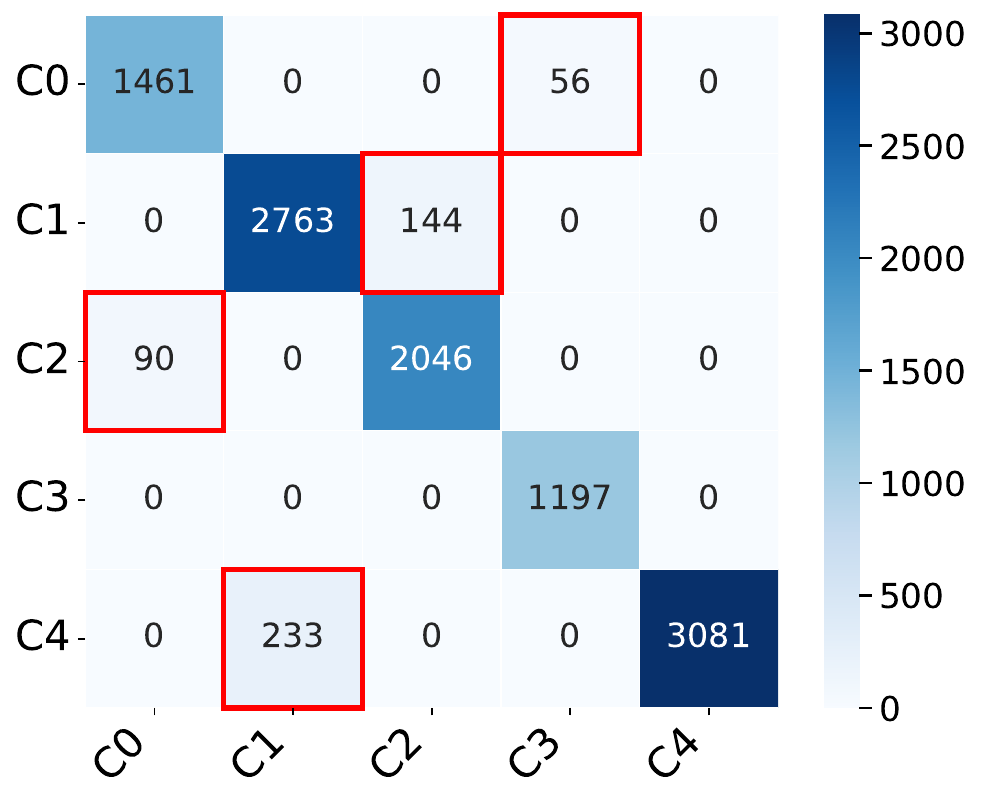}\hfill
  \includegraphics[width=0.24\linewidth]{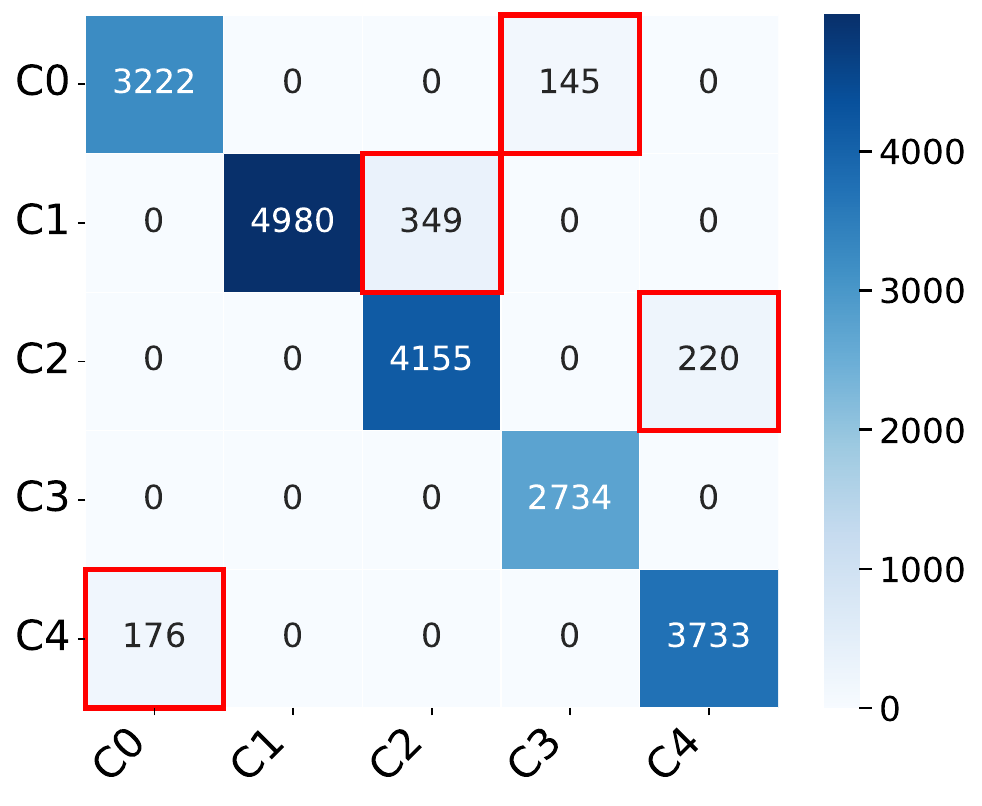}\hfill
  \includegraphics[width=0.24\linewidth]{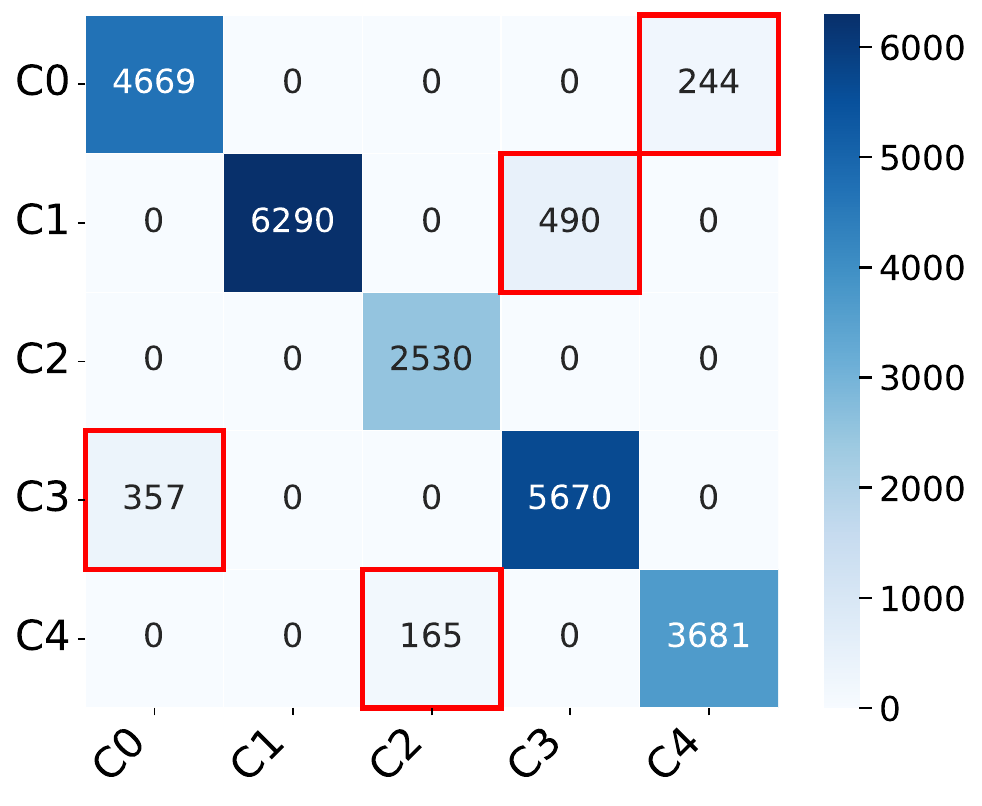}
  \caption{Inter-step cluster transition counts in raw KMeans-ID order on
  (left to right) \textsc{gsm8k}, \textsc{aqua}, \textsc{logiqa},
  \textsc{math}. The dominant off-diagonal cells differ across datasets
  because KMeans IDs are arbitrary; reordering by mean step position
  (Fig.~\ref{fig:trajectory_heatmap}) recovers a common structure.}
  \label{fig:trajectory_heatmap_raw_app}
\end{figure}

\paragraph{Phase-ordered transition matrices.}
After reindexing the clusters of each dataset using the
\textsc{Cluster}$\to$\textsc{Phase} maps in
Tabs.~\ref{tab:cluster_gsm8k}--\ref{tab:cluster_math}, we form the
phase-transition matrix
$P_{ij}=\Pr[\,\phi_{t+1}{=}j \mid \phi_{t}{=}i\,]$
on the reordered sequence $\phi_{1{:}T}$ of every trace.
All four matrices collapse to a strictly upper-bidiagonal absorbing form
(Fig.~\ref{fig:trajectory_heatmap}): a dominant dwell mass on the diagonal,
a single forward transition $P_{i,i+1}$, an absorbing terminal state
$P_{55}{=}1$, and skip-ahead probabilities $P_{i,j>i+1}$ that are zero up to
the resolution of the data.
Dwell probability rises monotonically from $P_1$ to $P_4$, indicating that
later phases are progressively harder to leave. This motivates gating the
intervention on observed self-loops (Section~\ref{sec:method:online}) rather
than intervening at every step: because $P_5$ is absorbing, a naive policy that
pushed the trace forward at every step would risk driving it into the
rumination tail. We note that the two-step self-loop gate we implement fires on
repeated self-loops but does not by itself prevent pushing a trace \emph{into}
$P_5$; avoiding that specifically is a limitation we return to in the Limitations.

\begin{figure}[ht]
  \centering
  \includegraphics[width=0.24\linewidth]{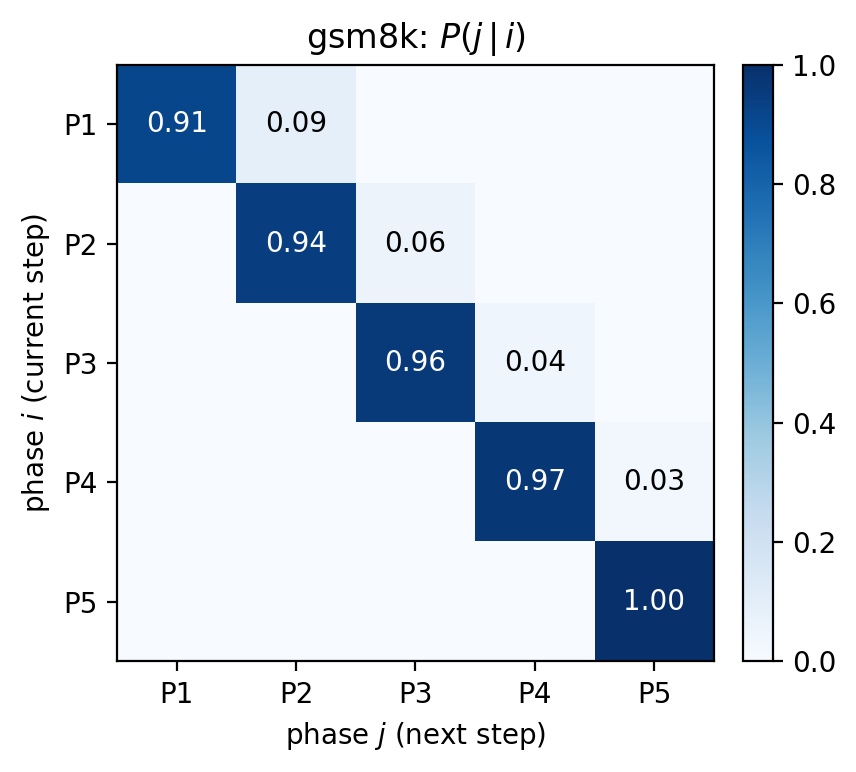}\hfill
  \includegraphics[width=0.24\linewidth]{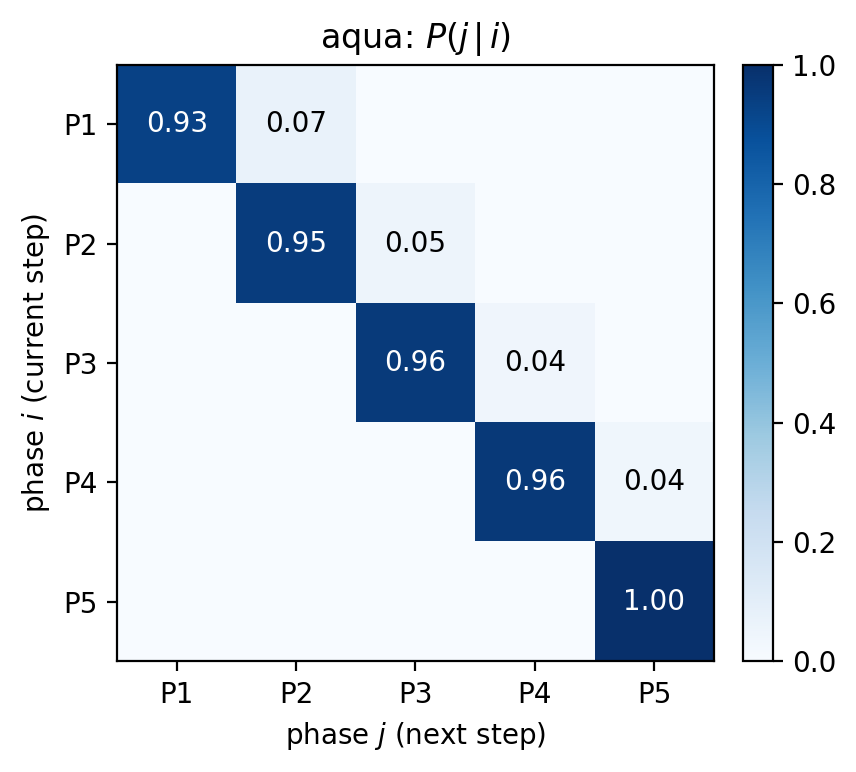}\hfill
  \includegraphics[width=0.24\linewidth]{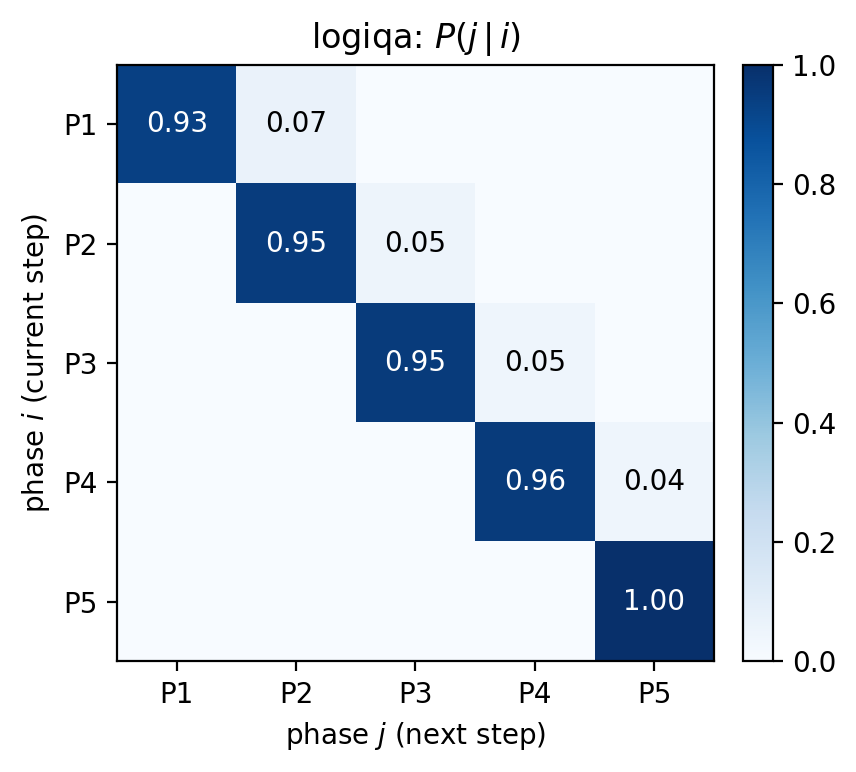}\hfill
  \includegraphics[width=0.24\linewidth]{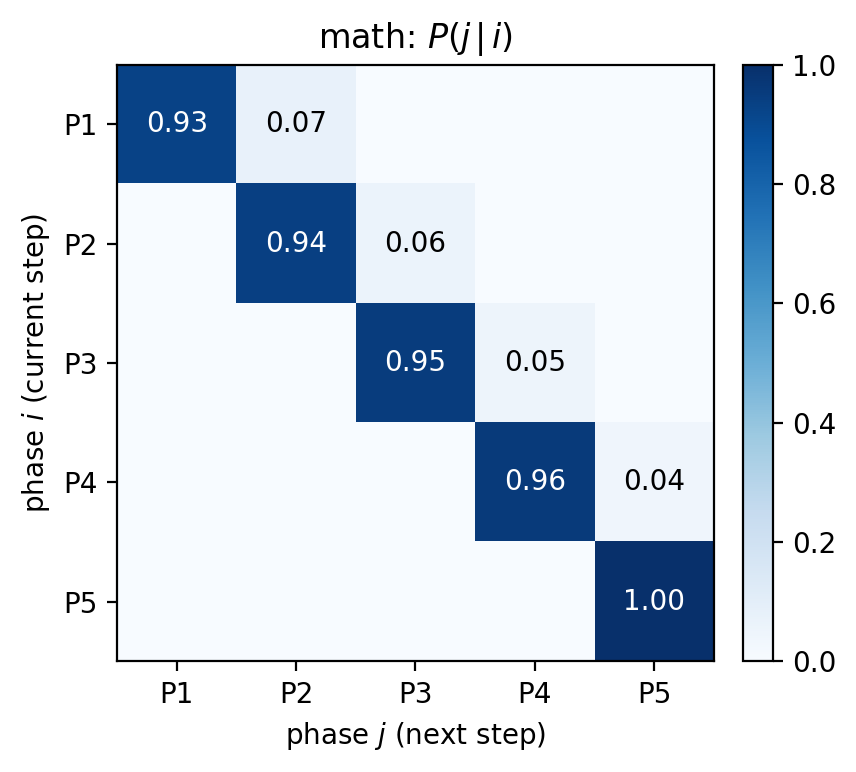}
  \caption{Phase-ordered transition matrix $P(j\,|\,i)$ on (left to right)
  \textsc{gsm8k}, \textsc{aqua}, \textsc{logiqa}, \textsc{math} after
  reindexing clusters by mean step position
  (Tabs.~\ref{tab:cluster_gsm8k}--\ref{tab:cluster_math}). All four matrices
  are upper-bidiagonal with an absorbing $P_5$, evidencing a near-deterministic
  $P_1{\to}P_2{\to}P_3{\to}P_4{\to}P_5$ trajectory.}
  \label{fig:trajectory_heatmap}
\end{figure}

\paragraph{t-SNE geometry.}
Figure~\ref{fig:trajectory_tsne} shows 2D t-SNE projections of the step
embeddings on each dataset, coloured by cluster. The clusters form
elongated, locally coherent ribbons rather than isotropic blobs, and adjacent
phases occupy adjacent regions of the manifold --- the geometric signature
of a one-dimensional progression rather than a categorical partition.
This is the figure referenced from the introduction
(Section~\ref{sec:transition-analysis}); its full per-dataset layout is
deferred here to keep the main text compact.

\begin{figure}[ht]
  \centering
  \includegraphics[width=0.24\linewidth]{figures/tsne_gsm8k.png}\hfill
  \includegraphics[width=0.24\linewidth]{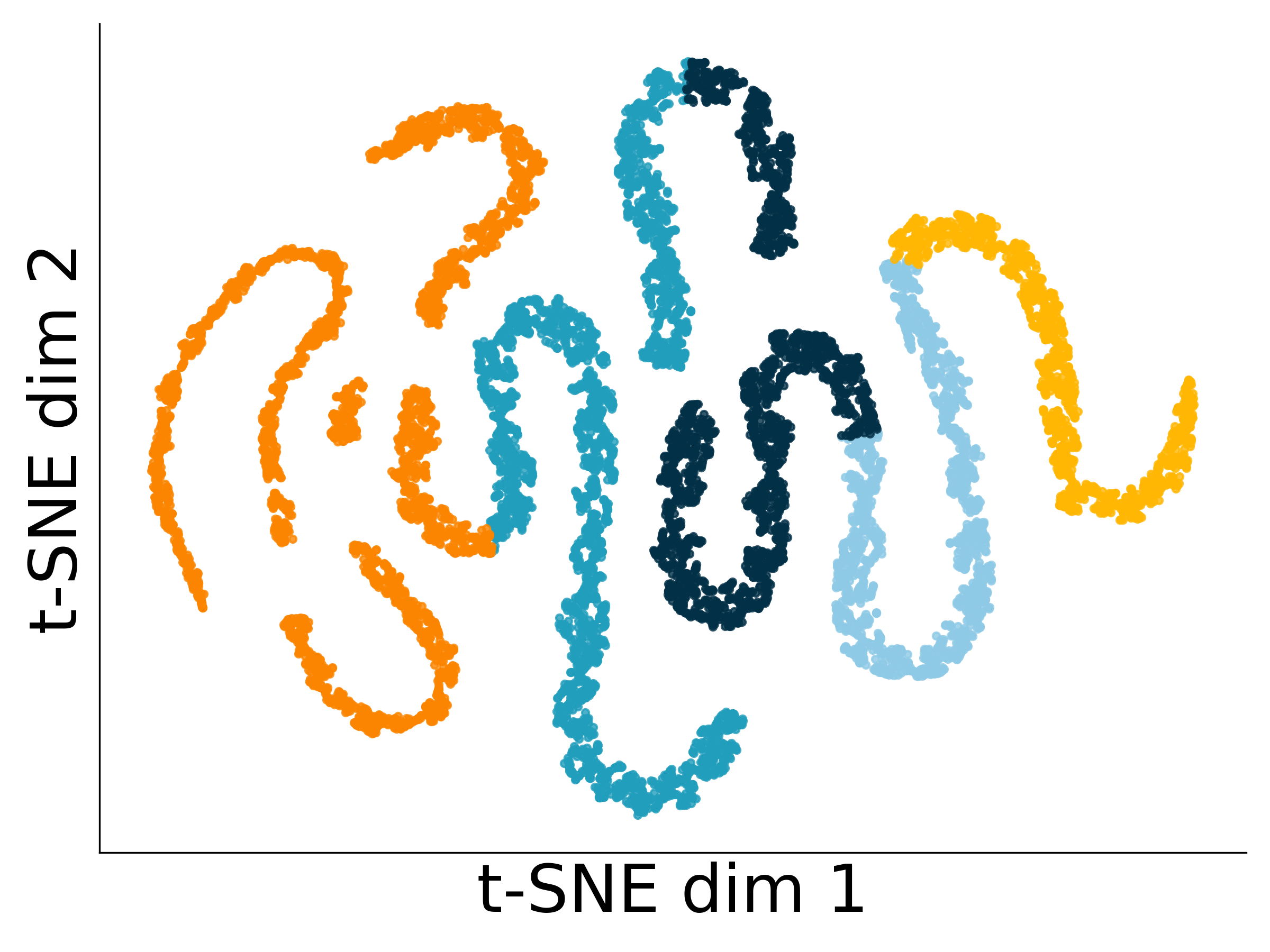}\hfill
  \includegraphics[width=0.24\linewidth]{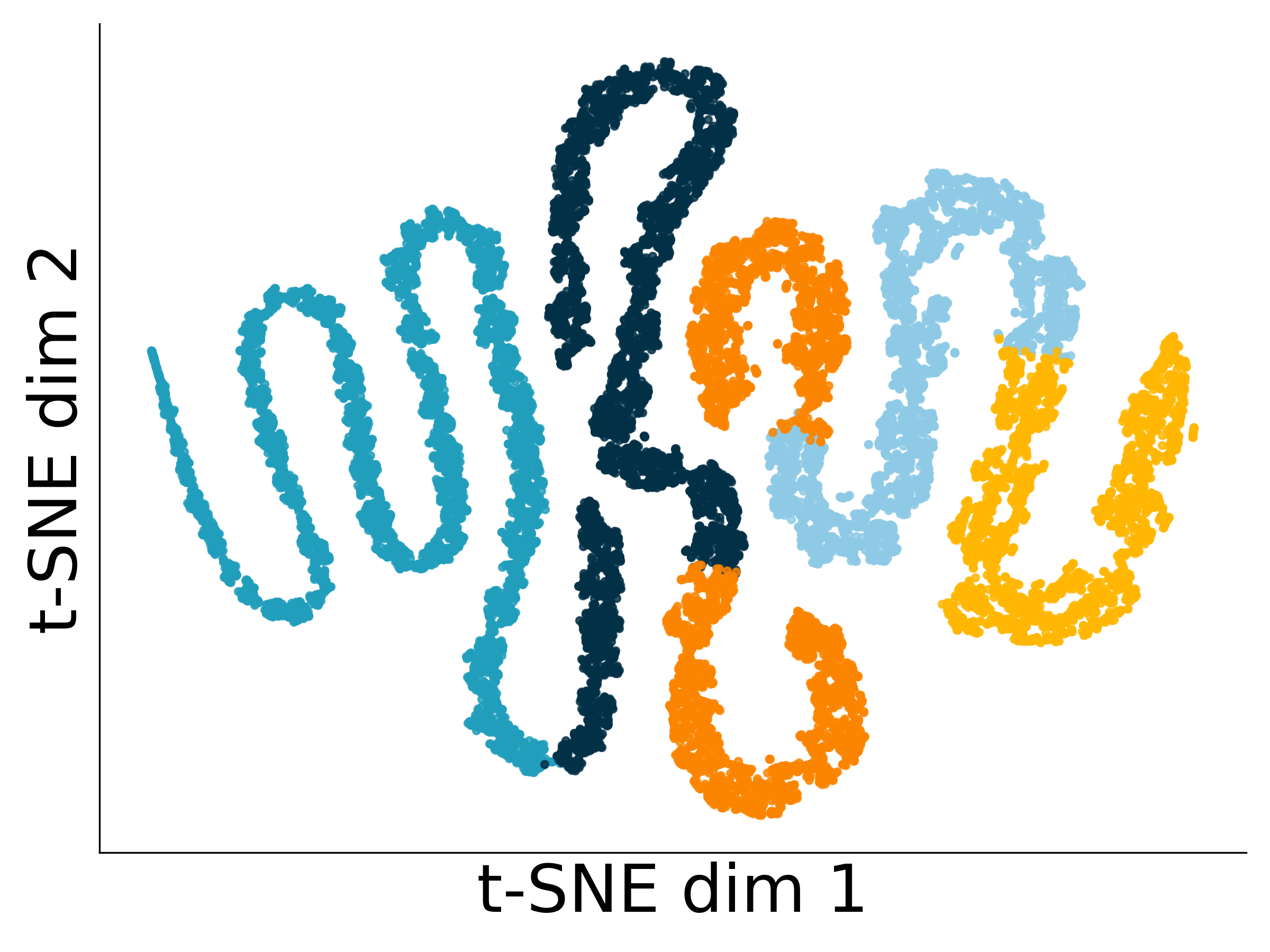}\hfill
  \includegraphics[width=0.24\linewidth]{figures/tsne_math.png}
  \caption{2D t-SNE of mean-pool last-layer features for reasoning
  steps, coloured by cluster, on (left to right) \textsc{gsm8k},
  \textsc{aqua}, \textsc{logiqa}, \textsc{math}. Clusters appear as
  contiguous ribbons that connect end to end, consistent with a
  one-dimensional reasoning trajectory.}
  \label{fig:trajectory_tsne}
\end{figure}

\paragraph{Quantitative summary.}
Table~\ref{tab:trajectory_stats} aggregates the four datasets along three
axes derived from the empirical transition counts:
the diagonal mass (probability a step \emph{stays} in its current phase),
the upper-triangular forward mass (probability of advancing), and the
lower-triangular backward mass (regression to an earlier phase).
Across all four datasets, dwell mass exceeds $0.93$, forward mass is scriptsize
but non-zero, and backward mass is zero in our data: no observed trace
re-enters an earlier phase. The terminal absorption probability
$P(P_5\,|\,P_5){=}1$ on these traces, so a trace that enters the rumination
tail does not leave it in our data.\footnote{The vanishing backward mass is
partly a consequence of phase ordering by mean step index; the non-trivial
empirical observations are (i) all skip-ahead transitions $P_{i,j>i+1}$ are
zero, and (ii) the absorption $P_{55}{=}1$ holds exactly on our traces.}

\begin{table*}[ht]
  \centering
  \scriptsize
  \begin{tabular}{lrrcccc}
    \toprule
    Dataset & traces & steps & dwell & forward & backward & $P(P_5\,|\,P_5)$ \\
    \midrule
    \textsc{gsm8k}  & 500 & 12{,}450 & 0.937 & 0.063 & 0.000 & 1.000 \\
    \textsc{aqua}   & 254 & 11{,}325 & 0.953 & 0.047 & 0.000 & 1.000 \\
    \textsc{logiqa} & 500 & 20{,}214 & 0.955 & 0.045 & 0.000 & 1.000 \\
    \textsc{math}   & 500 & 24{,}596 & 0.948 & 0.052 & 0.000 & 1.000 \\
    \bottomrule
  \end{tabular}
  \caption{Phase-trajectory statistics on four reasoning benchmarks.
  ``Dwell'' is total diagonal mass of the phase-transition matrix,
  ``forward'' is upper-triangular mass, and ``backward'' is lower-triangular
  mass after reindexing clusters by mean step position. All four datasets
  exhibit a strictly forward chain with an absorbing terminal phase.}
  \label{tab:trajectory_stats}
\end{table*}

\paragraph{Takeaway.}
Together, the per-dataset cluster semantics
(Tabs.~\ref{tab:cluster_gsm8k}--\ref{tab:cluster_math}), the contrast between
raw and phase-ordered heatmaps
(Figs.~\ref{fig:trajectory_heatmap_raw_app},~\ref{fig:trajectory_heatmap}), the
ribbon-like t-SNE geometry, and the zero-backward / unit-absorption statistics
indicate that reasoning traces of $\mathcal{M}_{\text{emb}}$ are well summarized by a
near-deterministic five-phase ordering---in part by construction, since phases are
ordered by mean step position (see footnote above)---rather than moving freely between
qualitatively distinct reasoning modes. This supports the
crosser-minus-stayer interventions of Section~\ref{sec:method:vectors},
$h \leftarrow h + \alpha\,v_{c}^{(\ell)}$: because transitions are sparse and
forward, the crosser-minus-stayer direction has a well-defined geometric
meaning at each phase, and because $P_5$ is absorbing, self-loop gating
(Section~\ref{sec:method:online}) is preferable to always-on intervention, which could
settle the controller into the rumination tail.

\section{Hyperparameters and Implementation Details}
\label{app:hyperparams}
We use $K = 5$ clusters per dataset (KMeans, seed $42$, on z-scored
mean-pool features). The patched layer set $\mathcal{L}$ is the upper
half of the decoder for the scriptsize models (layers $18$--$35$ for
Qwen3-4B-Thinking-2507, $17$--$34$ for Gemma-4-E2B) and the mid-stack
for Qwen3-32B-Thinking ($18$--$35$ out of $64$ layers); deeper-stack
patching follows the rationale from prior work that mid-to-late layers
carry the abstract conceptual content most relevant to reasoning
behaviour. The intervention strength $\alpha$ is tuned per model:
$\alpha = 0.05$ for the two Qwen-Thinking models, and $\alpha = 0.025$
for Gemma-4-E2B, above which generations on Gemma degenerate into token
loops -- a useful sanity signal that the patched offset is exerting real
effect on the residual stream rather than being absorbed as noise. All
other choices ($K$, seed, layer-range scheme, every-token patch mode)
are fixed across datasets and models.

\end{document}